%% file: root.tex
\documentclass[lettersize,journal]{IEEEtran}

\usepackage{amsthm}
\usepackage{amsmath,amssymb,amsfonts}
\usepackage{bm}

\usepackage{cite}
\usepackage{url}
\usepackage[bookmarks=true]{hyperref}

\usepackage{algorithm}
\usepackage[noend]{algpseudocode}

\usepackage{graphicx}
\graphicspath{{figure/}{./}}
\usepackage[caption=false,font=footnotesize,labelfont=rm,textfont=rm]{subfig}

\usepackage{booktabs}
\usepackage{multirow}
\usepackage{tabularx}
\usepackage{lscape}
\usepackage{adjustbox}
\usepackage{stfloats}
\usepackage{balance}

\usepackage{xspace}
\usepackage{color}
\usepackage[table,xcdraw]{xcolor}
\usepackage{comment}

\newtheorem{theorem}{Theorem}

\newtheorem{proposition}{Proposition}
\newtheorem{definition}{Definition}

\newtheorem{assumption}{Assumption}
\newtheorem{remark}{Remark}

\newcommand{\blue}{\color{black}}
\newcommand{\red}{\color{black}}

\newcommand{\brown}{\color{brown}}

\newcommand\abbrTSP{HPP\xspace}
\newcommand\abbrTSPPT{HPP-PT\xspace}
\newcommand\abbrRPT{RPT*\xspace}
\newcommand\abbrBRPT{F-RPT*\xspace}
\newcommand\abbrHATS{HATS\xspace}

\begin{document}

\title{RPT*: Global Planning with Probabilistic Terminals for Target Search in Complex Environments}

\author{Yunpeng~Lyu$^{1}$,
        Chao~Cao$^{2}$~\IEEEmembership{Member,~IEEE,}
        Ji~Zhang$^{2}$,
        Howie~Choset$^{2}$,~\IEEEmembership{Fellow,~IEEE,} and~Zhongqiang~Ren$^{1\dagger}$,~\IEEEmembership{Member,~IEEE}%
\thanks{
This work was supported by the Natural Science Foundation of China under Grant 62403313.
\emph{$^{\dagger}$Corresponding author: Zhongqiang Ren. email: zhongqiang.ren@sjtu.edu.cn)}}%
\thanks{$^{1}$Yunpeng Lyu and Zhongqiang Ren are with the Global College, Shanghai Jiao Tong University, Shanghai 200240, China (email: lv\_hn\_promising@sjtu.edu.cn; zhongqiang.ren@sjtu.edu.cn).}%
\thanks{$^{2}$Chao Cao, Ji Zhang, and Howie Choset are with the Robotics Institute, Carnegie Mellon University, Pittsburgh, PA 15213, USA (email: ccao1@andrew.cmu.edu; zhangji@andrew.cmu.edu; choset@andrew.cmu.edu).}}

\markboth{IEEE Transactions on Robotics}%
{Lyu \MakeLowercase{\textit{et al.}}: RPT*: Global Planning with Probabilistic Terminals for Target Search in Complex Environments}

\maketitle

\begin{abstract}
\input{abstract}
\end{abstract}

{\red }

\section{Introduction}
\label{TSPPT:sec:intro}
\input{intro}

\section{Related Work}
\label{TSPPT:sec:related}
\input{related}

\section{Problem Statement of \abbrTSPPT}
\label{TSPPT:sec:problem}
\input{problem}

\section{\abbrRPT Algorithm}
\label{TSPPT:sec:method}
\input{method}

\section{Systems and Applications}
\label{TSPPT:sec:system}
\input{system}

\section{Experimental Results}
\label{TSPPT:sec:result}
\input{result}

\section{Conclusion and Future Work}
\label{TSPPT:sec:conclusion}
\input{conclusion}

\appendices
\input{appendix_a}
\appendices
\input{appendix_b}

\bibliographystyle{IEEEtran}
\bibliography{reference}

\begin{IEEEbiography}[{\includegraphics[width=1in,height=1.25in,clip,keepaspectratio]{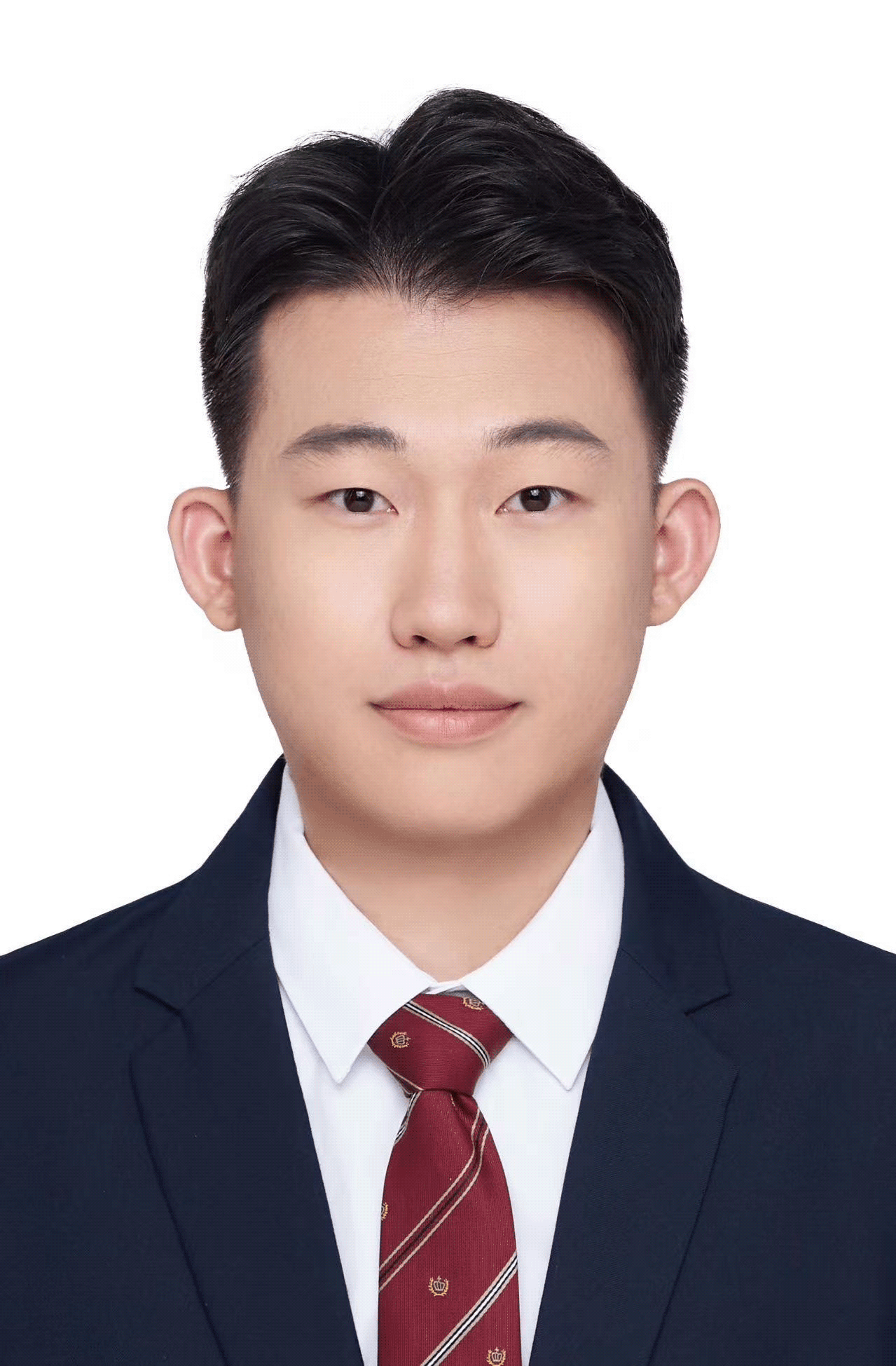}}]
    {Yunpeng Lyu} received the B.S. degree in mechanical engineering from Huazhong University of Science and Technology, Wuhan, China, in 2024.
    He is currently working toward the Ph.D. degree in control science and engineering with Shanghai Jiao Tong University, Shanghai, China.
    His research interests include robotics, target search, and embodied artificial intelligence.
\end{IEEEbiography}

\begin{IEEEbiography}[{\includegraphics[width=1in,height=1.25in,clip,keepaspectratio]{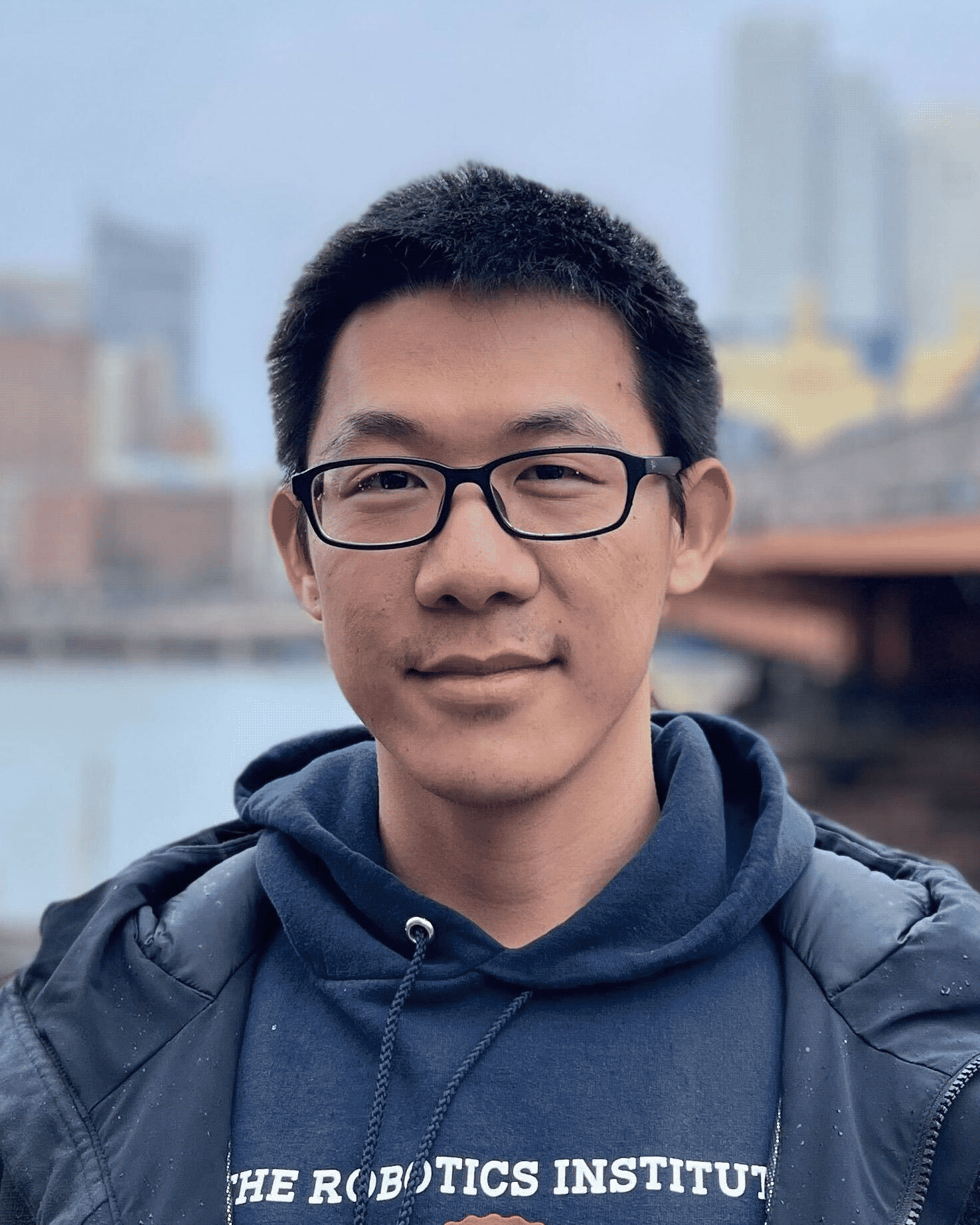}}]
    {Chao Cao} (Member, IEEE) received the B.S. degree in Computer Science from The University of Hong Kong, Hong Kong, in 2017, and the M.S. degree in robotics from the Robotics Institute, Carnegie Mellon University, Pittsburgh, PA, USA, in 2019. He is currently working toward the Ph.D. degree in robotics with the Robotics Institute, Carnegie Mellon University.
    His research interests include path and motion planning, mobile manipulation, autonomous exploration, and multi-robot coordination.
\end{IEEEbiography}

\begin{IEEEbiography}[{\includegraphics[width=1in,height=1.25in,clip,keepaspectratio]{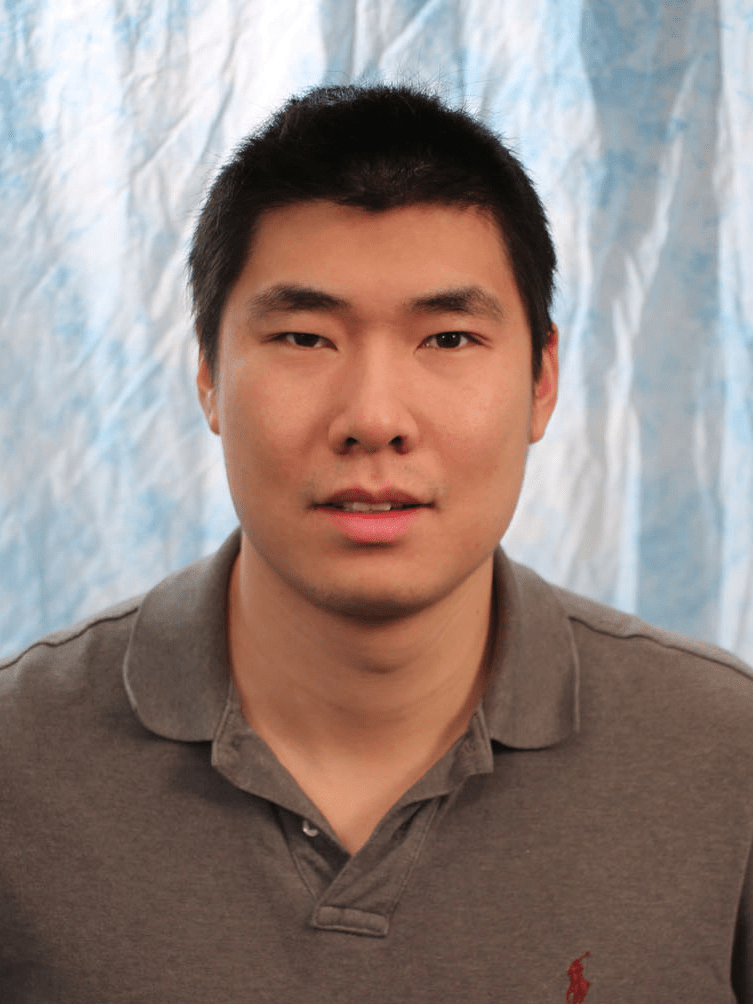}}]
    {Ji Zhang} received the Ph.D. degree in robotics from Carnegie Mellon University, Pittsburgh, PA, USA, in 2017.
    He is a Systems Scientist at the Robotics Institute at Carnegie Mellon University, where he leads in the development of a series of autonomous navigation algorithms. His work was ranked 1 on the odometry leaderboard of KITTI Vision Benchmark between 2014 and 2021.    
    He founded Kaarta, Inc., a CMU spin-off commercializing 3D mapping \& modeling technologies, and stayed with the company for 4 years as chief scientist.
    His research interests include robotic navigation, spanning localization, mapping, planning, and autonomous exploration.
\end{IEEEbiography}

\begin{IEEEbiography}[{\includegraphics[width=1in,height=1.25in,clip,keepaspectratio]{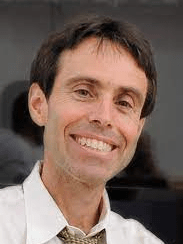}}]
    {Howie Choset} (Fellow, IEEE) received his B.S. Eng. degree in computer science and his B.S. Econ. degree in entrepreneurial management from the University of Pennsylvania (Wharton), Philadelphia, PA, USA, in 1990. He received M.S. and Ph.D. degrees in mechanical engineering from California Institute of Technology (Caltech), Pasadena, CA, USA, in 1991 and 1996, respectively.
    He is currently a Professor of Robotics at Carnegie Mellon University, Pittsburgh, PA, USA. His research group reduces complicated high dimensional problems found in robotics to low-dimensional simpler ones for design, analysis, and planning. 
\end{IEEEbiography}

\begin{IEEEbiography}[{\includegraphics[width=1in,height=1.25in,clip,keepaspectratio]{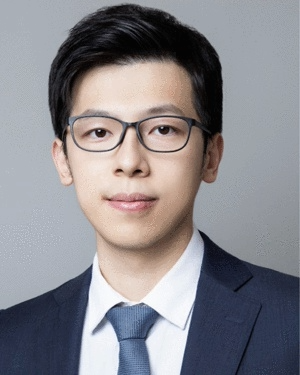}}]
    {Zhongqiang (Richard) Ren} (Member, IEEE) received the dual B.E. degree in mechatronics from Tongji University, Shanghai, China, and FH Aachen University of Applied Sciences, Aachen, Germany, in 2015, and the M.S. degree in mechanical engineering and the Ph.D. degree in mechanical engineering from Carnegie Mellon University, Pittsburgh, PA, USA, in 2017 and 2022, respectively.
    He is currently an Associate Professor with the Global College, Shanghai Jiao Tong University, Shanghai, China.
    His research interests include robotics, path and motion planning, and multi-robot systems.
    He is an Associate Editor for IEEE Transactions on Robotics, Autonomous Robots, ICRA and IROS.
\end{IEEEbiography}

\end{document}

%% file: abstract.tex
Routing problems such as Hamiltonian Path Problem (HPP), seek a path to visit all the vertices in a graph while minimizing the path cost.
This paper studies a variant, HPP with Probabilistic Terminals (\abbrTSPPT), where each vertex has a probability representing the likelihood that the robot’s path terminates there, and the objective is to minimize the expected path cost.
\abbrTSPPT arises in target object search, where a mobile robot must visit all candidate locations to find an object, and prior knowledge of the object’s location is expressed as vertex probabilities.
While routing problems have been studied for decades, few of them consider uncertainty as required in this work.
The challenge lies not only in optimally ordering the vertices, as in standard HPP, but also in handling history dependency: the expected path cost depends on the order in which vertices were previously visited.
This makes many existing methods inefficient or inapplicable.
To address the challenge, we propose a search-based approach \abbrRPT with solution optimality guarantees, which leverages dynamic programming in a new state space to bypass the history dependency and novel heuristics to speed up the computation.
Building on \abbrRPT, we design a Hierarchical Autonomous Target Search (HATS) system that combines \abbrRPT with either Bayesian filtering for lifelong target search with noisy sensors, or autonomous exploration to find targets in unknown environments.
Experiments in both simulation and real robot show that our approach can naturally balance between exploitation and exploration, thereby finding targets more quickly on average than baseline methods.

%% file: intro.tex
The well-known Hamiltonian Path Problem (\abbrTSP) seeks the shortest path that starts at a depot (start vertex) and visits every vertex in a graph exactly once.
{\blue In this paper, we study the Hamiltonian Path Problem with Probabilistic Terminals (\abbrTSPPT), a stochastic generalization of \abbrTSP, where each vertex is assigned a probability representing the likelihood that the robot’s path will terminate there, and the objective is to minimize the expected path length.} 
Our motivation for \abbrTSPPT comes from target search with a mobile robot.
Imagine one or more target objects, such as a key, left somewhere in a building, with the owner only vaguely recalling which rooms he had visited. In this setting, each possible location (e.g., a sofa or a table) is represented as a vertex in the \abbrTSP graph, with an associated probability indicating the likelihood of finding the target there.
The robot then plans its route by optimizing the visiting order of the vertices to minimize the expected path length. If the target is found at a vertex, the robot can terminate its search early; otherwise, it continues until all vertices have been visited.

\begin{figure}[t]
	\centering
	\includegraphics[width=\linewidth]{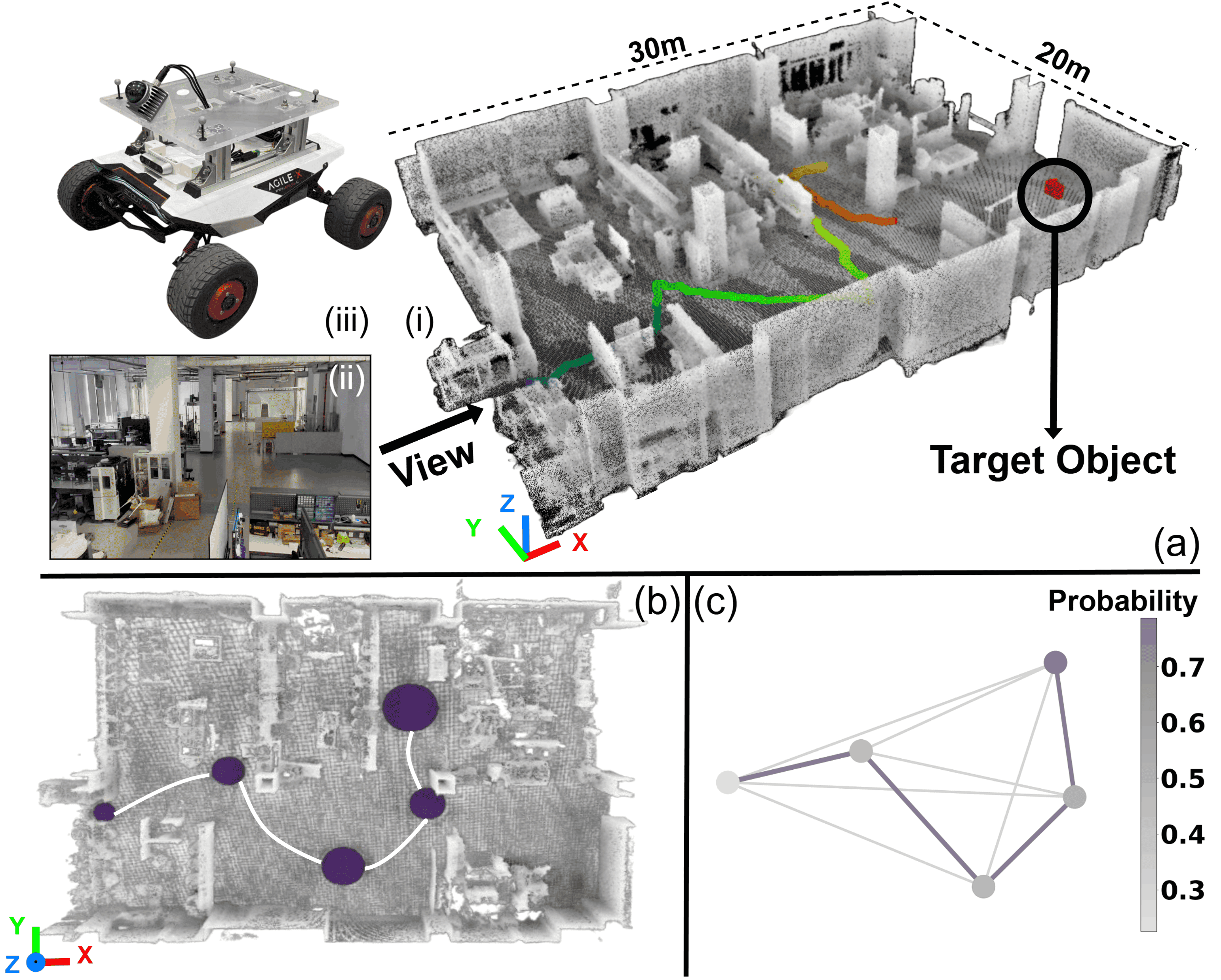}
    \caption{Target search in unknown environments using \abbrTSPPT. (a) shows the scenario of a real robot experiment, including (i) the point cloud map of the environment after finding the target object, where the colored line shows the trajectory of the robot; (ii) a photo of the test space; and (iii) the wheeled robot used in this experiment. (b) shows the top-down view of the environment, where purple dots represent candidate target locations, with marker sizes proportional to their corresponding probabilities of finding the target object there. (c) is the visualization of the generated solution path.}
    \label{tsppt:fig:intro}
\end{figure}

\abbrTSPPT generalizes \abbrTSP in the sense that if the probability at each vertex is set to zero, \abbrTSPPT reduces to the standard \abbrTSP. \abbrTSP is NP-hard~\cite{karp2009reducibility} and so is \abbrTSPPT.
Although routing problems, such as \abbrTSP and Traveling Salesman Problem (TSP), have been studied for decades with many variants, none of them considers uncertainty similar to \abbrTSPPT formulated in this paper.
The difficulty of \abbrTSPPT arises from two sources.
First, it inherits the fundamental challenge of \abbrTSP: determining the optimal visiting order of the vertices.
{\red Second, unlike \abbrTSP, the objective in \abbrTSPPT (i.e., the expected path length) depends on the visiting order of vertices in the path history.
Different visiting orders of the same set of vertices can yield different expected costs, which greatly enlarges the planning state space.
As a result, popular techniques for \abbrTSP and TSP such as integer linear programming \cite{applegate2006traveling,cacchiani2020models} via branch-and-cut, or Held-Karp dynamic programming~\cite{held1962dynamic,kimura2019accelerating} either cannot be readily applied to handle \abbrTSPPT or can lead to poor runtime efficiency when being naively adapted.
}

To address these challenges, the first contribution of this paper is proposing algorithms for \abbrTSPPT with solution optimality or bounded sub-optimality guarantees.
Naively applying search or dynamic programming to \abbrTSPPT can lead to a state space that includes all permutations of the vertices in the graph, which grows combinatorially with the number of vertices. 
To bypass such a curse of dimensionality, we develop a new approach called Routing with Probabilistic Terminals Star (\abbrRPT) which leverages heuristic search to iteratively plan paths from the depot towards other vertices.
\abbrRPT considers a new state space to bypass {\blue the dependency on the visiting order of vertices in the path history.
\abbrRPT further introduces a fast-to-compute and informative relaxation of the problem, which is used as the heuristic function to guide the planning process.
\abbrRPT also leverages new pruning rules to speed up the computation.
Compared to the well-known Held-Karp dynamic programming, \abbrRPT saves computation due to the guidance from the heuristic function and the pruning rules during planning.}
Besides, we observe that, as the number of visited vertices increases, the expected path cost grows slowly, which inspires us to develop a bounded sub-optimal variant of \abbrRPT by fusing focal search~\cite{4767270} to handle large graphs at the cost of finding solutions with small sub-optimality bounds.

We compared our \abbrRPT against several baselines {\blue including integer programming, Held-Karp dynamic programming, greedy method, popular routing solver LKH~\cite{helsgaun2000effective}, and adaptation from a recent learning method RL4CO~\cite{berto2023rl4co}, using both a public dataset on TSPLIB~\cite{reinelt1991tsplib} and a synthetic dataset with randomly generated graphs in Euclidean spaces.}
Results show that, within the same runtime limit of a minute, \abbrRPT can find optimal solutions in graphs with up to $40$ vertices, and find  $10\%$ sub-optimal solutions in graphs with up to $200$ vertices.
For the solution quality, the expected path length found by the baselines is up to three times that of our \abbrRPT.

{\blue To apply the proposed \abbrRPT algorithms to a practical robotic system, the second major contribution is a Hierarchical planning system for Autonomous Target Search (\abbrHATS).
\abbrHATS has three major components: (i) a module that converts the continuous workspace into discrete search regions, where each region corresponds to a vertex in the graph, and assigns the probability of finding targets at each vertex for \abbrTSPPT based on sensor reading, (ii) the high-level planning using \abbrRPT to determine the visiting order of target regions in the workspace, and (iii) the low-level planning for local trajectories satisfying kinodynamic constraints of the robot.}

To evaluate the proposed \abbrHATS system, we implemented \abbrHATS for two tasks: lifelong target search in a known environment with noisy sensors (\abbrHATS-L, L for lifelong), and exploratory target search in an unknown environment (\abbrHATS-U, U for unknown).
We compared \abbrHATS-L against baselines in Gazebo simulations and find that \abbrHATS-L tends to narrow down the likelihood about where target objects are with noisy sensors more quickly than the baselines.
Besides, we compared \abbrHATS-U against baselines in both simulation and real robots with YOLO software~\cite{yolo11_ultralytics} for real-time object detection.
Our \abbrHATS-U is more robust to misleading prior knowledge and can better balance between exploitation (moving toward regions with high target probability) and exploration (maintaining coverage over candidate regions), and thus find target objects with less time than the baselines on average.

The rest of the paper is organized as follows.
Sec.~\ref{TSPPT:sec:related} reviews the related work, and Sec.~\ref{TSPPT:sec:problem} formulates the \abbrTSPPT problem.
Sec.~\ref{TSPPT:sec:method} presents our \abbrRPT algorithm for \abbrTSPPT. 
Then, Sec.~\ref{TSPPT:sec:system} describes the \abbrHATS system for lifelong target search and target search in unknown environments, {\blue which explains how the \abbrRPT algorithm is embedded in the \abbrHATS system.}
Next, Sec.~\ref{TSPPT:sec:result} presents the experimental results and Sec.~\ref{TSPPT:sec:conclusion} discusses the conclusion and future work.
{\blue Since \abbrTSPPT is a new problem, we detail our integer programming and Held-Karp baseline in Appendix~\ref{TSPPT:app:exact}.
Finally, Appendix~\ref{TSPPT:app:hyperparam} shows the system hyperparameters.}

%% file: related.tex
\subsection{Hamiltonian Path Problem and its Variants}
The Hamiltonian Path Problem (HPP) and Traveling Salesman Problem (TSP) are both well-known routing problems that are NP-hard to solve to optimality~\cite{karp2009reducibility}.
While TSP seeks a tour that visits all vertices in a graph and returns to the depot (start vertex), HPP does not require the robot to return to the depot.
There are several variants of HPP based on whether the terminal vertex is fixed or not~\cite{monnot2005approximation,5991534}, and this paper does not require a fixed terminal vertex.

To solve HPP and TSP, a variety of routing approaches have been developed, trading off solution quality against runtime.
Among them, exact methods usually {\blue use dynamic programming such as the Held-Karp recursion for TSP and HPP~\cite{held1962dynamic}, or formulate the problem as an integer linear program and then solve by branch-and-cut techniques~\cite{cacchiani2020models,Behdani2014AnIA}.}
While being able to guarantee solution optimality, these exact approaches generally scale poorly to large graphs.
Bounded sub-optimal methods typically first solve a simpler problem and then convert the result to a desired solution path or tour for TSP and HPP~\cite{christofides2022worst}. 
{\blue
Besides, heuristic or learning-based methods~\cite{hussain2017genetic,min2024unsupervisedlearningsolvingtravelling}, including the Lin-Kernighan-Helsgaun (LKH) algorithm~\cite{helsgaun2000effective}, the Attention Model~\cite{kool2018attention}, and learning-based combinatorial optimization frameworks such as RL4CO~\cite{berto2023rl4co}, solve HPP and TSP quickly but usually have no solution quality guarantees.}
This paper is interested in exact algorithms.

Although HPP, TSP and their variants have been studied a lot, few of them consider uncertainty as required in this work for target object search scenarios.
Many of them consider uncertain edge costs.
For example, the Canadian Traveler Problem~\cite{bar1991canadian,aksakalli2016based} considers edges with stochastic traversability, which is only revealed upon arrival. 
Uncertainty has also been considered in stochastic variants of the TSP with randomness in travel costs~\cite{bertsimas1988probabilistic,toriello2014dynamic}, which focuses on uncertainty of cost rather than the vertex presence.
Other literature considers vertex uncertainty, which is still different from ours.
For example, Probabilistic Traveling Salesman Problem \cite{jaillet1985probabilistic,santini2022probabilistic} assigns each vertex a probability indicating the likelihood for the robot to skip that vertex during execution and seeks to minimize the expected path length, which is different from our target search problem as we require all vertices to be visited.
{\blue Stochastic Shortest Path problem~\cite{bertsekas1991analysis,guillot2020stochastic} seeks a policy to reach a pre-specified goal vertex subject to uncertainty when transiting from one vertex to another, while the \abbrTSPPT problem seeks a path to visit all vertices with uncertainty of terminating the path at each vertex.}
Overall, although some routing problems with uncertainty have been studied from different perspectives, the challenge of ensuring guaranteed coverage with uncertainty of target object existence remains unaddressed by the literature.

\subsection{Target Search}
Target search is an important problem in robotics that arises in applications such as search and rescue~\cite{lyu2023unmanned}, environmental monitoring~\cite{bayat_environmental_2017}, and hazardous operations~\cite{lin2025high}.
Target search has been studied in different ways, including search for static or dynamic targets~\cite{DADGAR201662,ishida1991moving}, and search in a global or local fashion.

Seemingly related yet fundamentally different problems include target tracking~\cite{kumar2021recent} and ergodic search~\cite{ren2023pareto}.
In target tracking, the target’s location is assumed to be known or partially observed at all times, and the objective is to maintain persistent estimates of its state as it moves~\cite{blackman2004multiple}. Similarly, ergodic search assumes access to a predefined spatial distribution encoding information density, and aims to generate trajectories whose visitation frequency matches this known distribution~\cite{mathew2011metrics}.
In both cases, the agent operates with explicit prior knowledge, and the challenge is to track targets rather than discovery of targets as in this work.

Target search has been investigated in a ``local'' planning fashion.
Existing work mostly focuses on how to extract information from human~\cite{ramrakhya2022habitat} or the environment~\cite{aydemir2013active} to guide the search of the target, and employs simple planning strategy such as greedy search~\cite{jia2019research} to iteratively navigate the robot to the most informative region.
Besides, recent work has incorporated semantics~\cite{app10020497} and contextual priors~\cite{chikhalikar2024semantic} to guide target search. Active Visual Search~\cite{pronobis2012large} approaches use probabilistic models relating objects, rooms, and scene context to compute a belief distribution over the target’s location and select views accordingly.
However, these approaches are usually limited to a relatively small region and the focus is to navigate in semantically complex environments~\cite{aydemir2013active} or interaction with human during object search~\cite{ramrakhya2022habitat}.
In contrast, our work complement these existing work by addressing the problem from a global perspective, and achieving efficient target search during the exploration of unknown environments or lifelong search.

\subsection{Exploration of Unknown Environments}
Another related topic to this work is the exploration of unknown environments, where the objective is to efficiently build a map or maximize coverage of unexplored space. 
These methods seek to explore the environment as quickly as possible, by selecting frontier points~\cite{yamauchi1998frontier}, maximizing information theoretic metrics~\cite{bai2016information}, or optimizing exploration paths via TSP~\cite{tare} and receding horizon strategies~\cite{nbvp} to ensure complete coverage of the environment.

Our \abbrHATS-U is inspired by recent hierarchical approaches~\cite{cao2023representation,zhou2023racer} that decompose exploration into high-level region selection and low-level motion planning, enabling efficient navigation and coverage in unknown spaces. 
Target search differs fundamentally from exploration as the goal is to locate a specific target as quickly as possible rather than the coverage of an unknown environment.
For target search, exploring every region is neither necessary nor desirable.

%% file: problem.tex
Let $G = (V, E)$ denote a finite directed complete graph with $N$ vertices, where the vertex set $V$ represents possible locations, and edge set $E \subseteq V \times V$ represents possible transitions between locations. 
Each edge is associated with a finite transition cost $c(e) \in \mathbb{R}^+$, which is the cost for the robot to go through the edge.
The edge costs satisfy the triangle inequality.
Each vertex $v\in V$ is associated with a real number $p(v) \in [0,1)$, denoting the probability that $v$ is used as the terminal, i.e., the robot ends its path at $v$.
In other words, $p(v)$ is the probability that the robot finds a target object at $v$, and the robot stops at $v$ thereafter.

\begin{definition}[Expected Path Cost]\label{tsppt:def:expected_cost}
Let $\pi=\pi(v_1, v_N)=(v_{1},v_{{2}},\dots,v_{N})$ denote a path from vertex $v_1$ to $v_N$ in $G$. 
{\blue Let $C(\pi)=\sum_{k=1}^{N-1} c(v_k,v_{k+1})$ denote the accumulated edge cost (i.e., path cost) along a path $\pi$ before finding a target.
Then, $C(\pi)$ is a random variable, and let $\xi(\pi)$ denote the expectation of path cost $C(\pi)$, which is defined as follows:}
\begin{align}
\label{tsppt:eq:expected_cost}
&\xi(\pi) := (1-p_1)p_2c_{1,2} + \\
&(1-p_1)(1-p_2)p_3(c_{1,2} + c_{2,3}) + \dots +  \nonumber \\ 
&(1-p_1)(1-p_2)\cdots(1-p_{N-1})(c_{1,2} + c_{2,3}+\dots + c_{N-1,N})\nonumber 
\end{align}
where $c_{k,k+1}$ denotes the edge cost $c(e),e=(v_k,v_{k+1})$ in $\pi$.
\end{definition}
Here, the $k$-th ($1\leq k \leq N-1$) term corresponds to the expected path cost of finding the target at the $(k+1)$-th vertex.

\begin{remark}\label{remark:last_vertex}
Note that, the last term in $\xi(\pi)$ does not have a $p_N$ in the middle of that term, since if the robot fails to find a target at the previous $N-1$ vertices, the robot has to go to the last vertex $v_N$ anyway, regardless of the probability of finding a target at $v_N$.
\end{remark}

\begin{definition}[\abbrTSPPT]
Let $v_s$ denote the start vertex, and the goal of the Hamiltonian Path Problem with Probabilistic Terminals is to find a path $\pi$ in $G$ starting from $v_s$ such that, (i) $\pi$ visits all vertices in $G$ exactly once, and (ii) the expected cost $\xi(\pi)$ reaches the minimum.
\end{definition}

\begin{remark}
When applying \abbrTSPPT to autonomous target search, there are different cases which may impose additional constraints to the problem formulation.
If there is only one target object (such as an ID card) in $G$ and it is known that the object must exist in $G$, then an additional constraint $\sum_{v\in V}p(v)=1$ should be imposed.
If there are multiple target objects (such as trash bins) in $G$ and the robot needs to find any one of them, then there is no additional constraint on $p(v),v\in V$.
Our approaches that will be presented next do not rely on those constraints and are applicable to all these problem variants.
\end{remark}

%% file: method.tex
A*-like algorithms~\cite{4082128} are widely used to solve a variety of path planning problems, and our presentation assumes readers are familiar with A*~\cite{4082128}.
However, A* cannot be directly applied to solve \abbrTSPPT since the expected path cost in Def.~\ref{tsppt:def:expected_cost} to be minimized depends on the visiting order of vertices.
A naive extension of A* for \abbrTSPPT would have to search in a state space that includes all permutations of vertices in $G$, which grows combinatorially with respect to the size of the graph, and will soon become computationally intractable.
To address this difficulty, we identify a transformation of the expected path cost in Def.~\ref{tsppt:def:expected_cost} by introducing an additional variable, which helps bypass this dependency on the visiting order of vertices and allows the planner to search a new state space efficiently.

\subsection{Concepts and Notations}
\subsubsection{New State Definition}
Along a path $\pi=(v_1,v_2,\cdots,v_N)$, let $q_k$ denote the probability of \emph{not} finding a target at any vertex among the first $k$ vertices in $\pi$:
\begin{align}
q_{k}= (1-p(v_1))(1-p(v_2))\cdots (1-p(v_k))
\end{align} 
With $q_{k}$, the expected cost $\xi(\pi)$ can be reformulated.

\begin{proposition}\label{tsppt:lem:xi_rewrite}
    The expected path cost in Eq. (\ref{tsppt:eq:expected_cost}) is the same as the following expression:
    \begin{align*}
        \xi(\pi) &= q_{1}c_{1,2} + q_{2}c_{2,3} + \dots +q_{N-1}c_{N-1,N} = \sum_{i=1}^{N-1} q_{i}c_{i,i+1}
    \end{align*}
\end{proposition}

\begin{proof}
Note that $q_{k-1} p_{k} + q_{k} = q_{k-1}(p_{k} + (1-p_{k})) = q_{k-1}$.
Hence for any $k=1,2,\cdots,N-1$ in Eq.~\eqref{tsppt:eq:expected_cost}, all terms related to $c_{k,k+1}$ can be written as
\begin{align*}
&(q_{k} p_{k+1} + q_{k+1} p_{k+2} + \cdots + q_{N-2}p_{N-1} + q_{N-1})c_{k,k+1} \\
&=(q_{k} p_{k+1} + q_{k+1} p_{k+2} + \cdots + q_{N-2})c_{k,k+1}\\
&=\cdots=q_k c_{k,k+1}.
\end{align*}

This can be applied to all terms $c_{k,k+1},k=1,2,\cdots,N-1$, thereby proving this proposition.
\end{proof}

We now define search states based on $q$ and show that the new state definition is Markovian, i.e., {\blue the future path depends only on the current state, independent from the past history including the set of visited vertices and their visiting orders.}
Let $s=(v,g,q,A)$ denote a state, which consists of a vertex $v\in V$, a $g$-cost as commonly used in A* $g \in \mathbb{R}^+$, a probability value $q \in [0,1)$, and $A\subseteq V$ {\blue representing the unordered set of vertices already visited by that partial path.}
We use $v(s),g(s),q(s),A(s)$ to denote the respective component of state $s$.
Later, we will use this state definition to search for paths in $G$ and use parent pointers to record the path.
Intuitively, each state $s$ identifies a path from the start vertex $v_s$ that has visited the vertices in $A$, and reaches the vertex $v(s)$ with the cost-to-come $g(s)$ and the probability $q(s)$ of not finding a target along that path yet.
\begin{remark}
Here, for a state $s$ and the corresponding path $\pi$ from the start vertex to $v(s)$ as recorded by the parent pointers, the $g$-cost $g(s)$ is same as $\xi(\pi)$ as defined in Proposition~\ref{tsppt:lem:xi_rewrite}.
We use two different notations to highlight that, $\xi$ is the expected cost of a solution path that visits all vertices as introduced in this new \abbrTSPPT, while $g(s)$ is the $g$-cost as commonly used in A* and is the expected cost of a (partial) path from the start vertex to some other vertex in $G$ before visiting all vertices.
\end{remark}

\subsubsection{State Expansion}
When expanding a state $s$, the path identified by $s$ is extended to all successor vertices in $G$ that have not yet been visited.
Let $s'$ denote a new successor state that is generated when expanding $s$.
Then, the components of $s'$ can be computed as follows: $v(s')$ is a successor vertex of $v(s)$, $v(s') \notin A(s)$ and
\begin{align}
g(s') &= g(s) + q(s) \cdot c(v(s),v(s'))\label{tsppt:eq:g_prime}\\
q(s')&=q(s) (1-p(v(s')))\label{tsppt:eq:q_prime}\\
A(s') &= A(s)\cup \{v(s')\}\label{tsppt:eq:A_prime}
\end{align}
Eq.~\eqref{tsppt:eq:g_prime} holds because of Proposition~\ref{tsppt:lem:xi_rewrite}.
Eq.~\eqref{tsppt:eq:q_prime} follows the definition of $q$, and 

\begin{proposition}\label{tsppt:lem:markov}
A state $s$ is Markovian when generating successor states by expanding $s$. 
\end{proposition}

\begin{proof}
When expanding $s$ to any successor $s'$ with $v(s') \notin A(s)$, Eq.~\eqref{tsppt:eq:g_prime}\eqref{tsppt:eq:q_prime}\eqref{tsppt:eq:A_prime} shows that the successor state $s'$ depends on $s$ only and there is no dependency on the history, i.e., the parents or ancestor states of $s$.
\end{proof}

Due to Proposition~\ref{tsppt:lem:markov}, we can iteratively extend paths in $G$ by expanding states based solely on their current $g,q,A$ variables, without considering their history. This eliminates the need to consider all possible permutations (i.e., visiting orders) of vertices during planning.

\subsubsection{State Comparison and Pruning}
With the new state definition, each state $s$ identifies a path from $v_s$ to $v(s)$ while visiting the vertices in $A(s)$ along the path.
For two paths represented by states $s_1,s_2$ that reach the same vertex $v$, (i.e., $v(s_1)=v(s_2)=v$), we need to compare them and potentially prune one of them to reduce the planning effort.
To compare these two states, we introduce the following dominance rule.
\begin{definition}[\bf State Dominance]\label{tsppt:def:state_dom}
    Given two states $s_1$, $s_2$ with the same vertex (i.e., $v(s_1) = v(s_2)=v$), we say $s_1$ dominates $s_2$, if all of the following conditions hold: (i) $g(s_1) \leq g(s_2)$
    and (ii) $A(s_1) \supseteq A(s_2)$.
\end{definition}

\begin{proposition}
    In Def.~\ref{tsppt:def:state_dom}, the condition $A(s_1) \supseteq A(s_2)$ implies that $q(s_1) \leq q(s_2)$.
\end{proposition}

\begin{proof}
For any state $s$, the value $q(s)$ is the product of $(1-p(v)),\forall v\in A(s)$ and $(1-p(v)) \in (0,1]$.
Given that $A(s_1)$ is a superset of $A(s_2)$, then $q(s_1)$ is equal to $q(s_2)$ multiplied by possibly more terms $(1-p(v)),v\in A(s_1), v\notin A(s_2)$, which makes $q(s_1) \leq q(s_2)$.
\end{proof}

If $s_1$ dominates $s_2$, then $s_2$ cannot lead to a better solution than $s_1$ and can thus be discarded.
We prove the correctness of this dominance rule later in Proposition~\ref{tsppt:lem:dom}, which involves some technical details.

If $s_1$ does not dominate $s_2$ and $s_2$ does not dominate $s_1$, then $s_1$ and $s_2$ are non-dominated by each other, and both of them need to be kept for planning.
Let $\mathcal{F}(v), \forall v\in V$ denote a \emph{frontier set} defined at each vertex, which is a set of non-dominated states at the same vertex (Fig.~\ref{tsppt:fig:label}).

Finally, let $h(s)$ denote the heuristic value of state $s$, which estimates the expected cost-to-go.
We provide a heuristic function in Sec.~\ref{tsppt:subsec:heuristic}.
{\blue The $f$-value of $s$ is defined as $f(s) := g(s) + h(s)$, where $g(s)$ is the accumulated expected cost from the start vertex to $s$, and $h(s)$ is the estimated remaining expected cost. $f(s)$ estimates the total expected cost of any solution that uses state $s$.}    

\begin{figure}
    \centering
    \includegraphics[width=\linewidth]{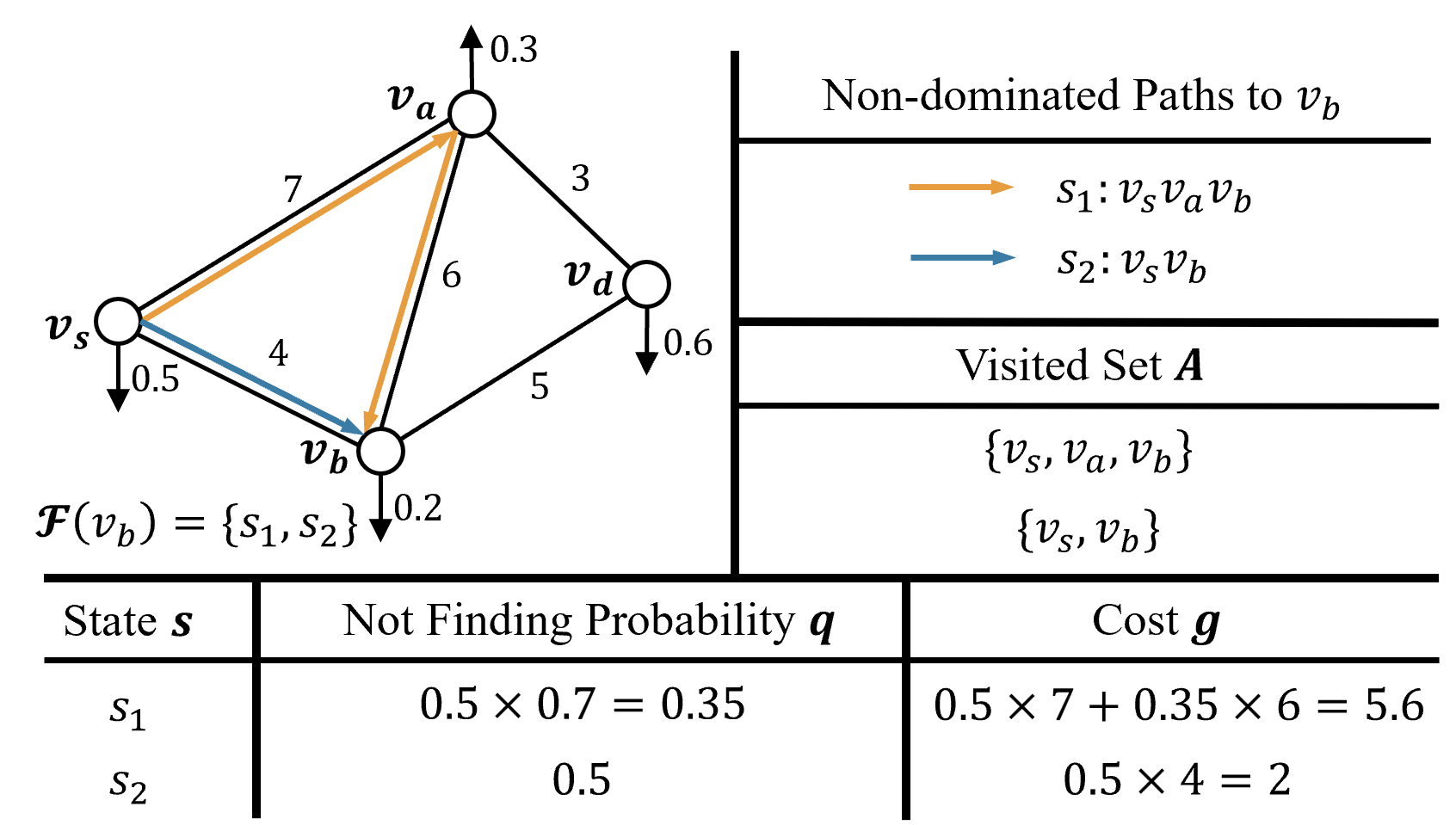}
    \vspace{-5mm}
    \caption{Visualization of related concepts.}
    \label{tsppt:fig:label}
\end{figure}

\subsection{Planning Algorithm}
\begin{algorithm}[tb]
		\small
	\caption{Pseudocode for \abbrRPT}\label{tsppt:alg:search1}
	\begin{algorithmic}[1]
            \State{$s_o \gets (v=v_s, g=0, q=1-p(v_s), A=\{v_s\})$}
            \State{$h(s_o) \gets $ \texttt{GetHeuValue}($s_o$), $f(s_o) \gets g(s_o)+h(s_o)$}\label{tsppt:alg:xxx:line_init}
            \State{Add $s_o$ to OPEN {\brown (and FOCAL)}}
            \State{$parent(s_o)\gets null$}
            \State{$\mathcal{F}(v) \gets \emptyset, \forall v\in V$}
            \While{OPEN $\neq \emptyset$}\label{tsppt:alg:search1:line_whileBegin}
            \State{extract $s$ from OPEN {\brown (or FOCAL)}}\label{tsppt:alg:search1:line_extract}
            \If{\texttt{IsPruned}($s$)}\label{tsppt:alg:search1:line_isPruneBeforeExp}
            \State{\textbf{continue}}
            \EndIf
            \State{\texttt{FilterAndAddFront}($s$)} 
            \If{$A(s)=V$}\label{tsppt:alg:search1:line_ifBetterIncum}
            \State{\textbf{return} \texttt{Reconstruct}($s$)}
            \EndIf
            
            \State{$U \gets$ \texttt{GetSuccessors}$(s)$}\label{tsppt:alg:search1:line_getSucc}
            
            \For{$v'\in U$}
            \State{$g' \gets g(s) + q(s) \cdot c(v(s),v')$}
            \State{$q' \gets q(s) (1-p(v'))$}
            \State{$A' \gets A(s) \cup \{v'\}$}
            \State{$s' \gets (v',g',q',A')$}
            \If{\texttt{IsPruned}($s'$)}\label{tsppt:alg:search1:line_isPruneAfterGen}
            \State{\textbf{continue}}
            \EndIf
            \State{$h(s') \gets$ \texttt{GetHeuValue}($s'$), $f(s') \gets g(s') + h(s')$}
            \State{$parent(s') \gets s$}
            \State{Add $s'$ to OPEN {\brown (and FOCAL)}}
            \EndFor
            
            \EndWhile \label{tsppt:alg:search1:line_whileEnd}
            
            \State{\textbf{return} failure}
  	\end{algorithmic}
        
    \hrule
	\begin{algorithmic}[1]
            \Statex{\texttt{IsPruned($s$)}}
            \For{$s' \in \mathcal{F}(v(s))$}
            \If{$s'$ dominates $s$}\label{tsppt:alg:isPruned:lineDom}
            \State{\textbf{return} true}
            \EndIf
            \EndFor
            \State{\textbf{return} false}
  	\end{algorithmic}
    
    \hrule
	\begin{algorithmic}[1]
            \Statex{\texttt{FilterAndAddFront($s$)}}
            \For{$s' \in \mathcal{F}(v(s))$}
            \If{$s$ dominates $s'$}\label{tsppt:alg:FilterAndAddFront:lineDom}
            \State{remove $s'$ from $\mathcal{F}(v(s))$}
            \EndIf
            \EndFor
            \State{add $s$ to $\mathcal{F}(v(s))$}
  	\end{algorithmic}
    
\end{algorithm}

To initialize, \abbrRPT (Alg.~\ref{tsppt:alg:search1}) first creates an initial state $s_o=\bigl(v_s,\,0,\,1-p(v_s),\,\{v_s\}\bigr)$, which is added to OPEN, a priority queue where states are prioritized based on their $f$-values from the minimum to the maximum.
\abbrRPT maintains parent pointers of the states during planning to easily reconstruct a path when a solution is found, and initializes the parent pointer of the initial state to be $null$.
The frontier sets $\mathcal{F}(v), \forall v\in V$ are initialized as empty sets.

In each search iteration (Line \ref{tsppt:alg:search1:line_whileBegin}-\ref{tsppt:alg:search1:line_whileEnd}), a state with the smallest $f$-value is extracted from OPEN for processing.
Then, a procedure \texttt{IsPruned} is called to check if $s$ should be discarded based on dominance. Specifically, \texttt{IsPruned} iterates all existing states $s' \in \mathcal{F}(v(s))$ and checks if $s'$ dominates $s$.
If so, $s$ cannot lead to a better solution than $s'$ and is thus discarded.
The current search iteration ends.

If $s$ is not pruned, $s$ is then added to $\mathcal{F}(v)$ in procedure \texttt{FilterAndAddFront}, where any existing states in $s' \in \mathcal{F}(v(s))$ is removed if it is dominated by $s$.
Notably, \texttt{FilterAndAddFront} differs from \texttt{IsPruned} as \texttt{FilterAndAddFront} seeks to remove the existing states in $\mathcal{F}(v(s))$ while \texttt{IsPruned} seeks to remove $s$.
Subsequently, the algorithm checks if $s$ visits all vertices in $G$.
If $A(s)= V$, all vertices are visited, and a solution is found. The path is then reconstructed by tracking the parent pointers of $s$ iteratively until reaching the initial state $s_o$.

If $s$ does not visit all vertices $(A(s)\neq V)$, $s$ is expanded.
The set $U$ of all unvisited successor vertices of $s$ is obtained.
For each of these successor vertices $v' \in U$, a corresponding new state $s'$ is generated as described in Eq.~\eqref{tsppt:eq:g_prime}\eqref{tsppt:eq:q_prime}\eqref{tsppt:eq:A_prime}.
Then, \texttt{IsPruned} is called to check if the new state $s'$ should be pruned.
Finally, if $s'$ is not pruned, the $f$-value of $s'$ is computed, and $s'$ is added to OPEN for future expansion.

\begin{remark}
    The proposed \abbrRPT is related to recent multi-objective search problems~\cite{HERNANDEZ2023103807,2024_AIJ_EMOA,ren22emoa} and their variants{\blue~\cite{cao2024heuristic,ren22mosipp,2025_SOCS_BOTSPTW_ShizheZhao,fu2023asymptotically,peer2025generalizing,ren22mopbd}.
    \abbrTSPPT is not a conventional multi-objective search problem, as its objective is a scalar value indicating the expected path cost rather than an additive cost vector.
    As a result, the heuristic design and pruning rules, which will be presented next, are very different from those existing work.}
\end{remark}

\subsection{Heuristic Computation}\label{tsppt:subsec:heuristic}

\begin{figure}
    \centering
        \includegraphics[width= 0.85\linewidth]{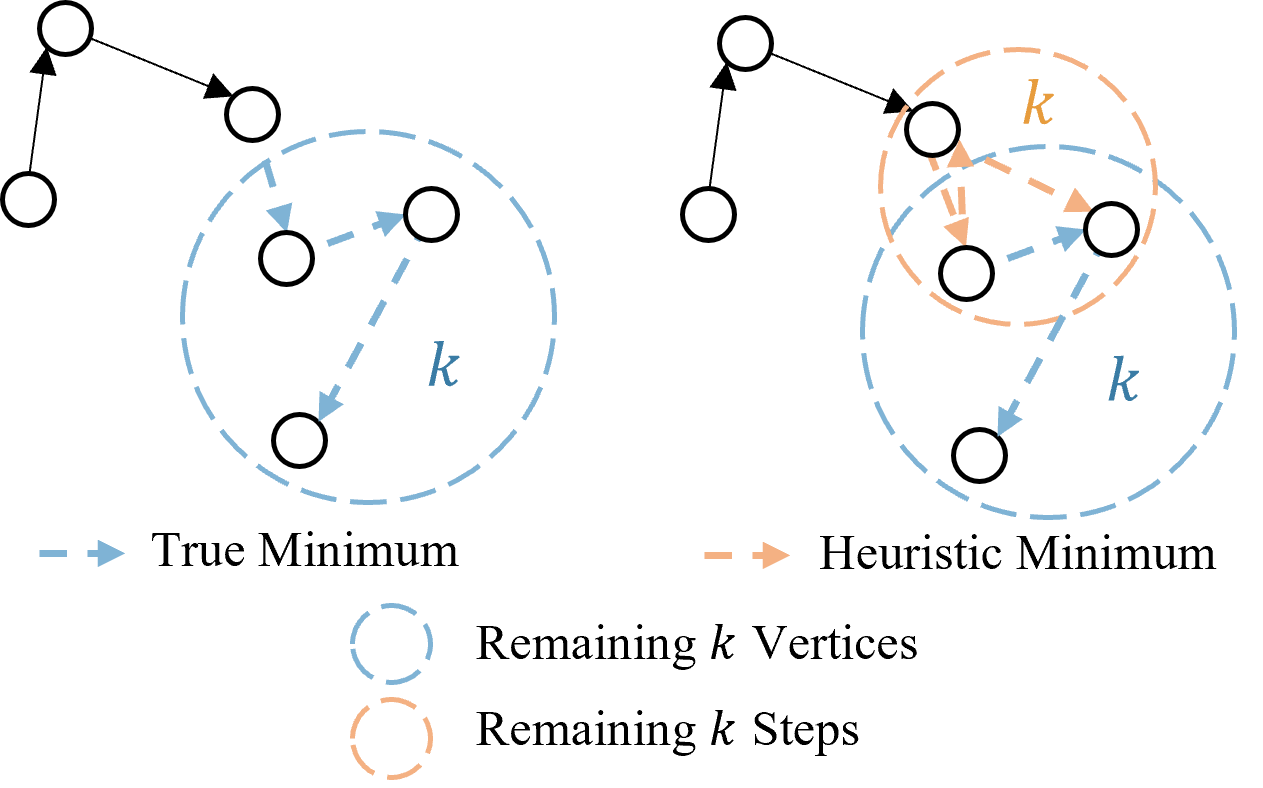}
    \caption{Heuristic function of our \abbrRPT algorithm: for the remaining $k$ unvisited vertices, the heuristic models the process as $k$ steps, explicitly allowing for repeated visits to vertices. }
    \label{tsppt:fig:Method_heuristic}
\end{figure}

This section presents a relaxation of the \abbrTSPPT problem, which results in a heuristic function that can be computed in polynomial-time. 
Our heuristic requires the pre-computation of a table $\gamma$ as follows.
For any $v\in V$, let $\gamma(v,k)$ denote an estimated cost-to-go from $v$ while visiting other remaining $k$ vertices where $k$ is an integer ranging from $0$ to $|V|-1$.
For the base case $k=0$, let $\gamma(v,0) = 0$, which means there is no additional vertices to be visited from $v$.
Then, for integers $i\in[0,k-1]$,
\begin{eqnarray}
\gamma(v,i+1)=\min_{u,u\neq v} (1-p(v))\cdot(c(v,u) + \gamma(u,i)), \label{tsppt:eqn:heurstic1} 
\end{eqnarray}
which is minimum incremental cost to visit any other vertex $u$ that is different from $v$ multiplied by the probability $(1-p(v))$ of not finding the target at $v$.

\begin{proposition}
    The pre-computation of table $\gamma$ has the worst case runtime complexity $O(|V|^2)$.
\end{proposition}
\begin{proof}
The table $\gamma$ involves both $v \in V$ and $i \in [0,|V|-1]$, and has $|V|^2$ items in total.
Eq.~\eqref{tsppt:eqn:heurstic1} allows the use of dynamic programming to compute the entire table from $i=0$ to $|V|-1$, and for each $i$, the computation needs to loop over all $v \in V$, which has runtime complexity $O(|V|^2)$ in the worst case.
\end{proof}

Then, the heuristic function $h(s)$ of any state $s$ can be computed during the search with the help of this pre-computed table $\gamma$ as follows.
\begin{align}
h(s) &= (q(s) / \left( 1 - p(v(s)) \right)) \cdot \gamma(v(s), k(s)), \nonumber\\
k(s) &= |V| - |A(s)|
\label{tsppt:eqn:heurstic2}
\end{align}
Here, the term $q(s)/(1-p(v(s)))$ is basically the $q$-value of the parent state of $s$ if $s$ is not the initial state ($s \neq s_o$).
When $s=s_o$, $q(s_o)=1$ and the term is still well-defined. The term $k(s)$ is the number of unvisited vertices of state $s$.

We will prove that the heuristic $h(s)$ is admissible (Proposition~\ref{tsppt:lem:heu_admiss}).
Intuitively, this heuristic allows a vertex with the cheapest visiting cost to be used many times in the future, i.e., it only requires that $v$ is not the same as its immediate successor $u$, and ignores the constraints that $u$ should not be the same as any other vertices that will be visited in the future.

\subsection{Bounded Sub-optimal Variant}

Finding an optimal solution for \abbrTSPPT is computationally burdensome for large graphs.
In large graphs, as the path length grows during planning, the $q$-values of states (Eq.~\eqref{tsppt:eq:q_prime}) keep decreasing, which makes the $g$-values of states plateau (Eq.~\eqref{tsppt:eq:g_prime}).
As a result, an optimal solution path may have negligibly smaller expected path cost than many other sub-optimal solutions, especially in large graphs.
We therefore develop a bounded sub-optimal variant of \abbrRPT using focal search~\cite{pearl1982studies} technique, and call this variant Focal Routing with Probabilistic Terminals Star (\abbrBRPT).

Focal search is a bounded sub-optimal variant~\cite{pearl1982studies} of A*, which further introduces a FOCAL list in addition to the OPEN list to store candidate states.
Let $f_{min}$ denote the minimum $f$-value of states in OPEN, and let $\epsilon \in [0,\infty)$ denote a sub-optimality bound.
FOCAL list is a priority queue containing states in OPEN whose $f$-value are within range $[f_{min},(1+\epsilon)f_{min}]$, and those states are prioritized based on a secondary heuristic function $h_F$.
In each search iteration, instead of selecting the first state from OPEN, focal search selects the first state from FOCAL for processing.

In \abbrBRPT, we define $h_F(s):= |V|-|A(s)|$ as the number of remaining unvisited vertices of a state, and FOCAL sorts the states based on $h_F$ from the minimum to the maximum.
In each search iteration, a state with the minimum $h_F$ value in FOCAL is selected for expansion.
By doing so, \abbrBRPT prefers states that have visited more vertices, and tends to find a solution earlier.
Since FOCAL is always a fraction of OPEN with $f$-value ranging between $[f_{min},(1+\epsilon)f_{min}]$, the resulting solution is guaranteed to be bounded sub-optimal (Theorem~\ref{tsppt:thm:bounded_suboptimal}), with the expected path cost $g$ no greater than $(1+\epsilon)g^*$, where $g^*$ is the true optimum of \abbrTSPPT.

\subsection{Analysis}
This section proves the solution optimality and bounded sub-optimality of \abbrRPT and \abbrBRPT.
To this end, we first prove the admissibility of the heuristic function in Proposition~\ref{tsppt:lem:heu_admiss}, and then the correctness of dominance pruning in Proposition~\ref{tsppt:lem:dom}.
Putting these together within the A* search framework, we show that the completeness and solution optimality of \abbrRPT.

\begin{proposition}\label{tsppt:lem:heu_admiss}
    (Admissible Heuristic). The heuristic in Eq. \eqref{tsppt:eqn:heurstic2} is admissible.
\end{proposition}

\begin{proof}
Let $h^*(s)$ denote the true minimum cost of any path from a given state $s$ to a goal state.
then, for any state $s$, the condition for admissibility is $h(s)\leqslant h^*(s)$.
For $k\geq 1$,
\begin{align*}
    \gamma(v(s),k)&=(1-p(v))\cdot\min_{u,u\neq v}(c(v,u) + \gamma(u,k-1))\nonumber\\
    &\leq (1-p(v))\cdot\min_{u,u\notin A(s)}(c(v,u) + \gamma(u,k-1)).
\end{align*}

Along path $\pi$, let $q_{m}$ and $q_{m-1}$ denote the probability of \emph{not} finding a target at the first $m$ and $(m-1)$ vertices in $\pi$ respectively. Then,
\begin{align*}
    h(s)&=q_{m-1}\cdot\gamma(v(s),k)\nonumber\\
    &\leq q_m\cdot\min_{u,u\notin A(s)}(c(v,u) + \gamma(u,k-1))\nonumber\\
    &\leq q_m\cdot\min_{u,u\notin A(s)}\{c(v,u) + (1-p(u))\cdot\\
    &\;\;\;\;\;\;\;\; \min_{u',u'\notin A(s)\cup \{u\}}(c(u,u')+...)\} = h^*(s)
\end{align*}

Since the condition $h(s)\leqslant h^*(s)$ holds for all states s, the heuristic function is admissible.
\end{proof}

To prove the correctness of the dominance rule in Def.~\ref{tsppt:def:state_dom}, we need the following assumption.
{\blue This assumption is only needed to complete the proof and is discussed in Remark~\ref{remark:assume}.}
\begin{assumption}\label{tsppt:assume:first_visit}
For a path $\pi$ that visits a vertex $v\in V$ more than once, when computing the $g$-cost, the probability of finding a target for the first visit $p_\text{first}(v)$ is the same as the probability value $p(v)$ as defined in the problem statement, and the probability of finding a target at $v$ after the first visit $p_\text{after}(v)$ is always zero. 
\end{assumption}

To prove the correctness of the dominance rule, we use the following proof-level convention for paths that arise only through conceptual concatenation in the proof. Such repeated-visit paths are not generated by RPT*.

\begin{proposition}\label{tsppt:lem:dom}
    For two states $s_1,s_2$ with $v(s_1)=v(s_2)$, if $s_1$ dominates $s_2$, $s_2$ cannot lead to a better solution than $s_1$.
\end{proposition}
\begin{proof}
    As illustrated in Fig.~\ref{tsppt:fig:Method_cutshort}, let $\pi_a(s_1),\pi_a(s_2)$ denote the path from $v_s$ to $v(s_1)$ and $v(s_2)$ respectively, and let $\pi_b(s_1)$ denote an arbitrary path from $v(s_1)$ that visits all unvisited vertices not in $A(s_1)$.
    In other words, $\pi_b(s_1)$ completes the path $\pi_a(s_1)$.
    Similarly, let $\pi_b(s_2)$ denote an arbitrary path from $v(s_2)$ that traverses all unvisited vertices not in the set $A(s_2)$, and the concatenated path $\pi=\pi_a(s_2) + \pi_b(s_2)$ is a solution.
    Let $\pi' = \pi_a(s_1) + \pi_b(s_2)$ denote a path that concatenates $\pi_a(s_1)$ and $\pi_b(s_2)$ together.
    Since $s_1$ dominates $s_2$, we know $g_1 \leq g_2$ and $\xi(\pi') \leq \xi(\pi)$.
    We need to show that for any possible $\pi'$, there exists a no-worse path $\pi'' = \pi_a(s_1) + \pi_b(s_1)$ than $\pi'$ (i.e., $\xi(\pi'')\leq \xi(\pi')$), and as a result, there is no need to search $\pi$ since $\xi(\pi'')\leq \xi(\pi') \leq \xi(\pi)$.
    
    To show that, since $A(s_1) \supseteq A(s_2)$, there can be vertices that are visited more than once in $\pi'$.
    For each of those vertices $u_k$ visited more than once in $\pi'$, if $u_k$ is the last vertex in $\pi$, then $u_k$ can be simply removed from $\pi'$ with a strictly lower expected cost than $\xi(\pi')$. Consequently, the new expected cost is less than $\xi(\pi)$.
    Otherwise ($u_k$ is not the last vertex in $\pi'$), let $u_{k-1}$ and $u_{k+1}$ denote the previous and the next vertex of $u_k$ in $\pi$.
    Then, a new path $\pi''$ can be obtained by shortcutting $u_k$, i.e., in $\pi''$, the vertices $u_{k-1}$ and $u_{k+1}$ are directly connected and $u_k$ is skipped (Fig.~\ref{tsppt:fig:Method_cutshort}).
    The resulting shortcut path $\pi''$ must have no-larger expected path cost than $\xi(\pi')$ because of Assumption~\ref{tsppt:assume:first_visit} and that the graph edge costs satisfy the triangle inequality.
    Specifically, using the expression in Def.~\ref{tsppt:def:expected_cost}, 
    \begin{align}
        \xi(\pi') &= q_1c_{1,2} + q_2c_{2,3} + \cdots + q_{k-1}c_{k-1,k} + q_{k}c_{k,k+1} + \nonumber\\
        &\;\;\;\;\;\;\;\; \cdots + q_{N-1}c_{N-1,N} \label{proof:eq8} \\
        &= q_1c_{1,2} + q_2c_{2,3} + \cdots + q_{k-1}(c_{k-1,k} + c_{k,k+1}) + \nonumber\\
        &\;\;\;\;\;\;\;\; \cdots + q_{N-1}c_{N-1,N} \label{proof:eq9}\\
        &\geq q_1c_{1,2} + q_2c_{2,3} + \cdots + q_{k-1}(c_{k,k+1}) + \nonumber\\
        &\;\;\;\;\;\;\;\; \cdots + q_{N-1}c_{N-1,N} = \xi(\pi'') \label{proof:eq10}
    \end{align}
    Here, Eq.~\eqref{proof:eq8} is equal to Eq.~\eqref{proof:eq9} since $u_k$ is visited more than once and $p_k = p_\text{after}=0$.
    As a result, $q_{k} = q_{k-1}(1-p_k) = q_{k-1}$.
    The inequality from Eq.~\eqref{proof:eq9} to Eq.~\eqref{proof:eq10} follows from the triangle inequality of the edge costs $c_{k-1,k} + c_{k,k+1} \geq c_{k-1,k+1}$.
    We can keep doing so until all the vertices that are visited more than once in $\pi$ are shortcut.
    As a result, the resulting path $\pi''$ after shortcut has $\xi(\pi'') \leq \xi(\pi')$, which means $\pi'$ can be discarded.
\end{proof}

\begin{figure}
    \centering
    \includegraphics[width=0.95\linewidth]{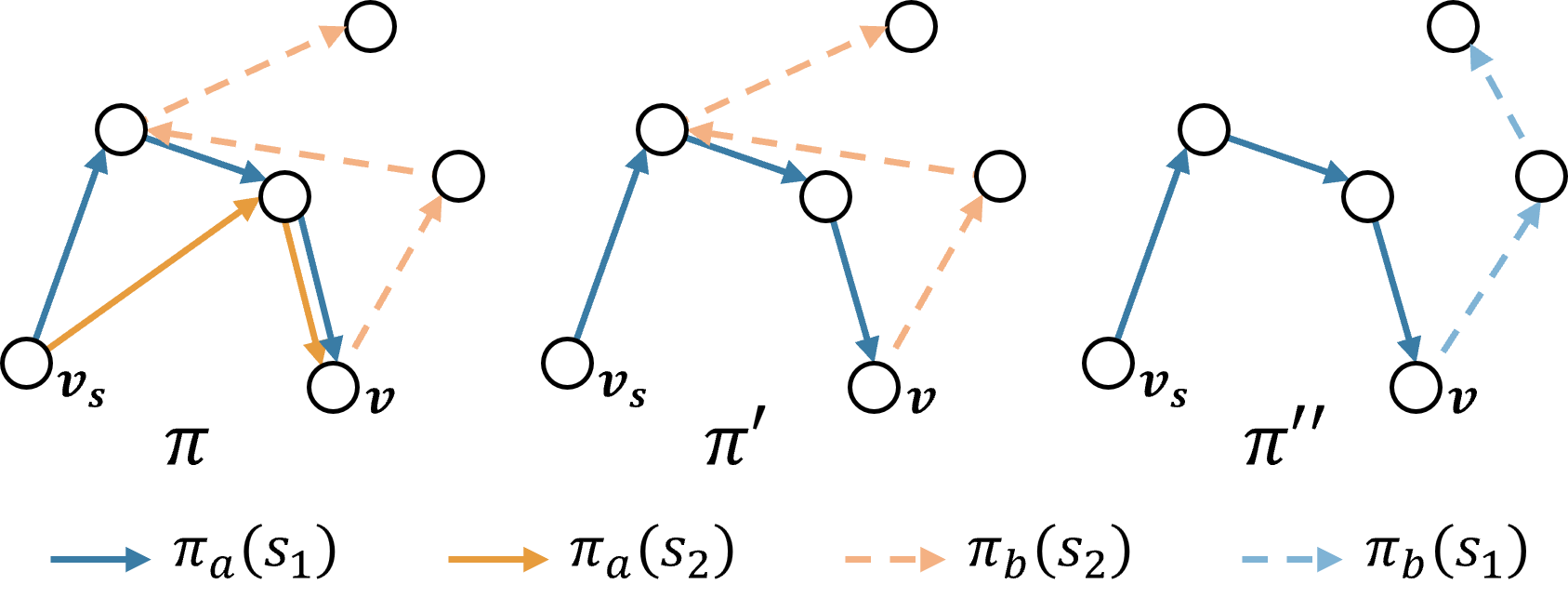}
    \caption{Illustration of shortcutting in the proof.}\label{tsppt:fig:Method_cutshort}
    \vspace{-5mm}
\end{figure}

\begin{remark}\label{remark:assume}
{\blue
The purpose of introducing Assumption~\ref{tsppt:assume:first_visit} is to complete the proof, since the proof involves paths that repeatedly visit the same vertices by construction.
The path cost with repeatedly visited vertices is undefined in the problem statement and Assumption~\ref{tsppt:assume:first_visit} therefore defines such costs.
In other words, paths with repeatedly visited vertices are not feasible solutions for the \abbrTSPPT problem and are never returned by  \abbrRPT in practice.
For the practical \abbrHATS system for lifelong target search as presented later in Sec.~\ref{TSPPT:sec:HATS-L}, the robot can repeatedly visit the same location as many times as needed.
Every time when \abbrRPT is called, the produced path has no repeated visit.
However, when \abbrRPT is called many times, it is possible that a newer path directs the robot to visit a vertex that was previously visited by an older path.
}
\end{remark}

The following theorems summarize the main properties of \abbrRPT and \abbrBRPT, whose correctness relies on the general best-first search framework (such as A* and its variants~\cite{pearl1982studies}).
We show the main ideas in the following proofs.

\begin{theorem}
    \abbrRPT is complete in a sense that, (i) for an unsolvable \abbrTSPPT instance, \abbrRPT terminates in finite time and reports failure; (ii) for a solvable \abbrTSPPT instance, \abbrRPT returns a feasible solution in finite time. 
\end{theorem}
\begin{proof}
\abbrRPT conducts systematic best-first search as A*~\cite{pearl1982studies}, where in each search iteration a state with the minimum $f$-value in OPEN is selected for expansion, and when expanding a state, \emph{all} possible successor states are considered.
During the search, a state $s$ is pruned only when there is another state $s'$ that dominates $s$ in either \texttt{IsPruned} or \texttt{FilterAndAddFront}.
By the definition of dominance (Def.~\ref{tsppt:def:state_dom}) and Proposition~\ref{tsppt:lem:dom}, such pruning only removes sub-optimal solutions since for any future solution from $s$, there must exist a corresponding (better) solution from $s'$.
As a result, for a solvable problem instance, \abbrRPT never eliminates the only path to any vertex $v \in G$ if a better path has not been found yet, and terminates the search with a feasible solution.
For an unsolvable problem instance, since the graph $G$ is finite, there is a finite number of possible paths.
\abbrRPT terminates after enumerating all paths and reports failure.
\end{proof}

\begin{theorem}
    \abbrRPT finds an optimal solution for \abbrTSPPT.
\end{theorem}
\begin{proof}
    \abbrRPT systematically enumerates all paths in the graph $G$, and by Def.~\ref{tsppt:def:state_dom} and Proposition~\ref{tsppt:lem:dom}, dominance pruning never eliminates a path along an optimal solution.
    With an admissible heuristic (Proposition~\ref{tsppt:lem:heu_admiss}), the $f$-value of any state $s$ popped from OPEN is a lower bound of the true optimal solution cost~\cite{pearl1982studies}, i.e., $f(s)\leq \xi_*$ with $\xi_*$ being the true optimal cost of an optimal solution.
    At termination, a state $s$ that has visited all vertices is popped.
    Since all vertices are visited and there is no additional vertex to be visited, $h(s)=0$ and $g(s) = f(s) \leq \xi_*$.
    As $\xi_*$ is the true optimum, it must be the case $g(s) = \xi_*$ and \abbrRPT finds an optimal solution.
\end{proof}

\begin{theorem}\label{tsppt:thm:bounded_suboptimal}
    Let $\xi_*$ denote the true optimal cost value for any solvable \abbrTSPPT problem instance, \abbrBRPT finds a $(1+\epsilon)$ bounded sub-optimal solution $\pi$ for \abbrTSPPT, i.e., $\xi(\pi) \leq (1+\epsilon)\xi^*$.
\end{theorem}
\begin{proof}
    The use of focal search~\cite{pearl1982studies} allows \abbrBRPT to select a state $s$ with $f$-value within range $[f_{min},(1+\epsilon)f_{min}]$ with $f_{min}$ being the minimum $f$-value of any state in OPEN.
    At termination, the popped state $s$ is guaranteed to have $g$-cost $g(s) = f(s) \leq (1+\epsilon) f_{min} \leq  (1+\epsilon) \xi_*$, which ensures the $(1+\epsilon)$ bounded sub-optimality of the returned solution.
\end{proof}

%% file: system.tex
{\blue The presented \abbrTSPPT planner operates on a finite graph, whereas a practical target search problem is defined in a continuous space with many viewpoints to be checked.
To apply our \abbrRPT and \abbrBRPT in practical target search applications, we propose the Hierarchical planning system for Autonomous Target Search (\abbrHATS), which is a hierarchical system that uses \abbrTSPPT to iteratively plan global paths to direct the robot, and integrates perception, mapping and other modules for real robot deployment.}
We showcase the use of \abbrHATS framework in two distinct application scenarios.

Section \ref{TSPPT:sec:HATS-L} details \abbrHATS-L for \emph{lifelong} target search with unreliable sensors in \emph{known} environments, which combines the Bayesian estimation approach~\cite{thrun_probabilistic_2005} to iteratively improve the belief about where the target object is during navigation.
Specifically, a prior map about the environment is known and an initial target probability distribution is available.\footnote{A uniform distribution is used when no clue where the target might be.}
The robot is equipped with a noisy sensor that is subject to false positive and false negative detection and is described by a probabilistic sensor model.
As a result, \abbrHATS needs to iteratively generate new \abbrTSPPT instances with the updated probability every time after a sensor detection, and calls \abbrRPT to re-plan the path.
    
Section \ref{TSPPT:sec:HATS-U} describes \abbrHATS-U for integrated exploration and target search, where the robot needs to simultaneously explore an \emph{unknown} environment and search for a target, terminating when the target is found.
Specifically, when no prior map is available, the robot has to autonomously explore the environment to incrementally build a map using SLAM algorithms while searching for the target.

\subsection{\abbrHATS System Overview}
\abbrHATS is a hierarchical system (Fig. \ref{tsppt:fig:Pipeline}) consisting of four modules: (i) Perception, (ii) Mapping, (iii) Global Planning, and (iv) Local Planning.

\begin{figure}[t]
	\includegraphics[width=\linewidth]{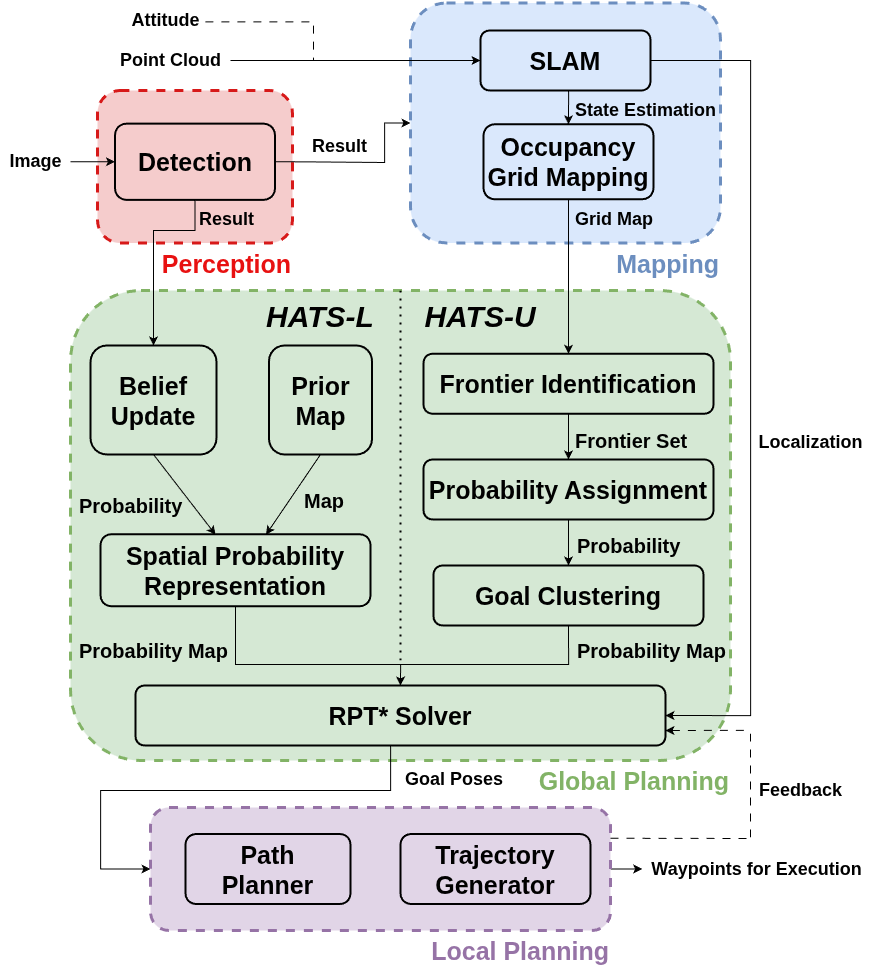}
 	\vspace{-5mm}
    \caption{The architecture and workflow of the Hierarchical planning system for Autonomous Target Search (\abbrHATS).
    HATS iteratively runs the workflow that begins by collecting sensory data and eventually outputs waypoints for execution. {\blue When local planning fails, \abbrHATS has a feedback loop that calls global planning to update the global path and search in different regions.}
    The two variants HATS-L and HATS-U differ in their global planning.}
    \label{tsppt:fig:Pipeline}
\end{figure}

\subsubsection{Perception}
The Perception module provides sensory data for target detection and for the subsequent Mapping module. 
We use a LiDAR and a camera as the primary sensors on the robot.
The LiDAR provides real-time point cloud data for mapping, localization, and navigation, while the camera captures images within a fixed Field Of View (FOV) for target detection. 
The Perception module continuously executes onboard target detection and verification using real-time image data from the camera.
In \abbrHATS-L, the perception module constantly updates the likelihood of finding a target at each location after each detection.
In \abbrHATS-U, this constant detection has an interruption mechanism: when the confidence score of the detected target exceeds a predefined threshold value, the robot immediately stops its motion and claims success.

\subsubsection{Mapping}
\label{tsppt:sec:mapping}
This module processes LiDAR point cloud data to generate both state estimations and map representations for the subsequent planning modules.
It runs a SLAM algorithm which produces both a point cloud map of the surrounding environment and an SE(2) pose estimate by using LiDAR (and IMU) data.
We adopt LIO-mapping~\cite{ye2019tightly} as the underlying SLAM algorithm, and other SLAM methods can also be applied within \abbrHATS.

We employ SLAM to construct an occupancy grid map from point clouds where each cell is marked as one of the following three labels: ``unknown'', ``occupied'' or ``free'', and ``occupied'' and ``free'' both belong to the category of ``known'', indicating whether the cell is occupied by an obstacle or not.
With the help of SLAM, we update the occupancy grid by projecting sensor measurements into the map’s coordinate frame and updating each cell’s label through ray tracing.
To handle point cloud noise, the occupancy status of each cell is determined by evaluating the density or ratio of points contained within its volume in a similar way as in~\cite{hornung2013octomap}.

This occupancy grid is the input to the subsequent Global Planning module.
As the robot explores, the occupancy grid grows incrementally until a target is found or the environment is fully explored without finding any target.

\subsubsection{Global Planning}
\label{tsppt:alg:global_planner}
This module generates a graph based on the occupancy grid map 
and computes the probability of each vertex in the graph to formulate an \abbrTSPPT problem, which is then solved to plan a global path $\pi$ for the robot.
Each vertex in this path $\pi$ corresponds to a SE(2) pose in the environment, which will be a goal pose for the subsequent local planning.
In other words, this global path determines the visiting order of goal poses in the environment.
The Global Planning module in \abbrHATS-L and \abbrHATS-U differs (Fig.~\ref{tsppt:fig:Pipeline}).

In \abbrHATS-L, an environment occupancy grid is built in advance.
A fixed graph $G$ embedded in the grid is provided, as well as the probability distribution of the vertices, which are used to formulate an \abbrTSPPT problem.
Every time the robot reaches a vertex, based on the sensor reading, the probability of that vertex is updated using the Bayesian rule (Sec.~\ref{TSPPT:sec:HATS-L}) and a new \abbrTSPPT is generated and solved to determine the next goal pose of the robot.

In \abbrHATS-U, the environment is unknown, and the robot has to incrementally build a point cloud map using a SLAM algorithm to generate the corresponding occupancy grid during navigation.
Based on the grid, the frontier points from the boundary between known and unknown spaces are extracted.
Then, as elaborated in Sec.~\ref{TSPPT:sec:HATS-U}, these frontiers are used to compute the probability distribution and clustered to generate a graph $G$, which leads to an \abbrTSPPT problem. Solving this problem yields the goal pose for the subsequent module.

\subsubsection{Local Planning}
While the Global Planning determines the sequence of goal poses $(x,y,\theta)$ to be visited, the Local Planning is responsible for both (i) planning the detailed geometric path to connect those goal poses and (ii) generating control command to follow the geometric path.
Here, we use A*-based search in the occupancy grid to plan the geometric path connecting the goal poses from the global planning, and we use the Dynamic Window Approach~\cite{fox2002dynamic} for local path following, which allows avoiding unexpected dynamic obstacles (such as pedestrians) in a local and reactive fashion.

{\blue Furthermore, if the Local Planning module fails to generate or execute a feasible trajectory to the assigned goal pose, it reports the failure to the Global Planning module. 
Consequently, the corresponding edge or vertex in the graph of \abbrTSPPT is then removed from the graph, and the global planner is subsequently invoked again to update the global path.
This feedback mechanism allows \abbrHATS to recover from local infeasibility rather than assuming that all graph edges remain traversable throughout execution.}

\subsection{Lifelong Target Search with Bayesian Update}
\label{TSPPT:sec:HATS-L}
For lifelong search with noisy sensors, as Fig.~\ref{tsppt:fig:Pipeline} shows, \abbrHATS-L employs a Bayesian filter\cite{thrun_probabilistic_2005} for probabilistic map update.
We first introduce the general form of the Bayesian filter and demonstrate its specific application within our experiment.
{\blue Let $b_t(v) \in [0,1)$ denote the estimated probability values of vertex $v \in V$ in the graph $G$ at time step $t$, which is the posterior belief of the target object being at $v$, i.e., $b_t(v) = b_t (\text{object is at } v)$.
Let $\overline{b}_t(v)$ denote the corresponding prior belief at time $t$ before incorporating the sensor measurement.
At the beginning of a global planning cycle, the vertex probability used by \abbrTSPPT is instantiated from the current belief,
$p_t(v)=b_t(v)$,
and remains fixed within that planning call. After the robot moves to the next target location, and receives a sensor measurement, the belief is updated and a new \abbrTSPPT instance is generated.
Note that $p(v)$ in Secs.~\ref{TSPPT:sec:problem} and \ref{TSPPT:sec:method} denotes the vertex probability used within one \abbrTSPPT instance, whereas $b_t(v)$ denotes the physical belief that evolves across closed-loop execution.}
Furthermore, let $u_t$ and $z_t$ denote the control input and measurement at time $t$, respectively. In our experiment, $u_t$ represents the robot's transition between vertices, which is an edge in the graph traversed by the robot.
The belief refinement process involves two steps:

\subsubsection{Prediction Step}
As formulated in Eq.~\eqref{tsppt:eq:prediction}, the prior belief $\overline{b}_t(v)$ is computed by incorporating the control input $u_t$ using the discrete Chapman-Kolmogorov equation\cite{thrun_probabilistic_2005}:
\begin{equation}
    \bar{b}_{t}(v) = \sum_{v' \in V} p(v \mid u_{t}, v') \cdot b_{t-1}(v'), \forall v \in G.
    \label{tsppt:eq:prediction}
\end{equation}
Here, $p(v \mid u_{t}, v')$ denotes the transition probability for the target object to move from $v'$ to $v$ when the robot takes control $u_t$.
Since the target object is static in our setting, the transition probability has a simple form:
\begin{align}
    &p(v \mid u_{t}, v') = 1, \text{if } v= v'\nonumber\\
    &p(v \mid u_{t}, v') = 0, \text{otherwise, i.e., }v\neq v',
\end{align}
regardless of the control $u_t$ taken by the robot.\footnote{A possible future work is to consider more complicated transition probability and extend our \abbrHATS to find dynamic target objects.}
Consequently, the prediction step has no effect on the belief, and the prior belief at time $t$ is the same as the posterior belief from the previous time step $t-1$: 
\begin{equation}
    \overline{b}_t(v) = {b}_{t-1}(v), \forall v\in G.
\end{equation}

\subsubsection{Update Step}
The posterior belief $b_t(v)$ corrects the prior $\overline{b}_t(v)$ with the current measurement $z_t$ via Bayes rule, when the robot reaches vertex $v$ after taking the control $u_t$:
\begin{equation}
    b_{t}(v) = \eta P(z_{t} \mid v) \cdot \bar{b}_{t}(v).
    \label{tsppt:eq:update}
\end{equation}
{\blue Here, $P(z_{t} \mid v)$ denotes the observation model,} and $\eta$ is the normalization factor so that the posterior belief sums to one.
Specifically, the sensor measurement $z_t\in\{0,1\}$ is binary and tells whether the target object is found ($z_t=1$) or not ($z_t=0$).
Let $\neg v$ denote that the target object is not at $v$, then the observation model is
\begin{align}
    &P(z_{t}=1 \mid v) = \alpha_1, &\forall v\in G \\
    &P(z_{t}=0 \mid v) = 1-\alpha_1, &\forall v\in G \\
    &P(z_{t}=1 \mid \neg v) = \alpha_2, &\forall v\in G\\
    &P(z_{t}=0 \mid \neg v) = 1-\alpha_2, &\forall v\in G,
\end{align}
where $\alpha_1,\alpha_2 \in [0,1]$ are real numbers indicating the true positive rate and false positive rate of the sensor, and the normalization factor is:
\begin{equation}
\eta = \frac{1}{P(z_t \mid v)\overline{b}_t(v) + P(z_t \mid \neg v)(1 - \overline{b}_t(v))}.
\end{equation}
This observation model indicates the likelihood for the sensor to return $0$ (or $1$) at any vertex $v\in G$ given that the object is at $v$ (or not at $v$).
After the update step, $b_t(v)$ denotes the probability of the target object being at $v$, i.e., $b_t(v) = b_t(\text{object is at } v)$.

\begin{remark}
If it is known that there is only one target, the update step (Eq.~\eqref{tsppt:eq:update}) needs to update the probability of all vertices. Intuitively, detecting the target object at one vertex $v$ will increase the probability value at $v$ while lowering the probability value at other vertices $u\neq v, u\in G$.
In our setting, we don't know how many targets are there. 
As a result, the update step (Eq.~\eqref{tsppt:eq:update}) only needs to be applied to the vertex $v$ that is reached by the robot after taking control $u_t$.
\end{remark}

\subsection{Target Search in Unknown Environments}
\label{TSPPT:sec:HATS-U}
As shown in Fig.~\ref{tsppt:fig:Pipeline}, \abbrHATS-U dynamically constructs the graph $G$ and updates probability distributions via the following three steps:
\subsubsection{Frontier Identification}
Based on the occupancy grid built by the Mapping module (Sec.~\ref{tsppt:sec:mapping}), \abbrHATS-U extracts frontiers, i.e., the boundary cells between unknown and free cells.
Note that the boundary cells between unknown and occupied cells are ignored since occupied space are obstacles and cannot be accessed by the robot.

\subsubsection{Probability Assignment}

\begin{figure}
    \centering
    \includegraphics[width=\linewidth]{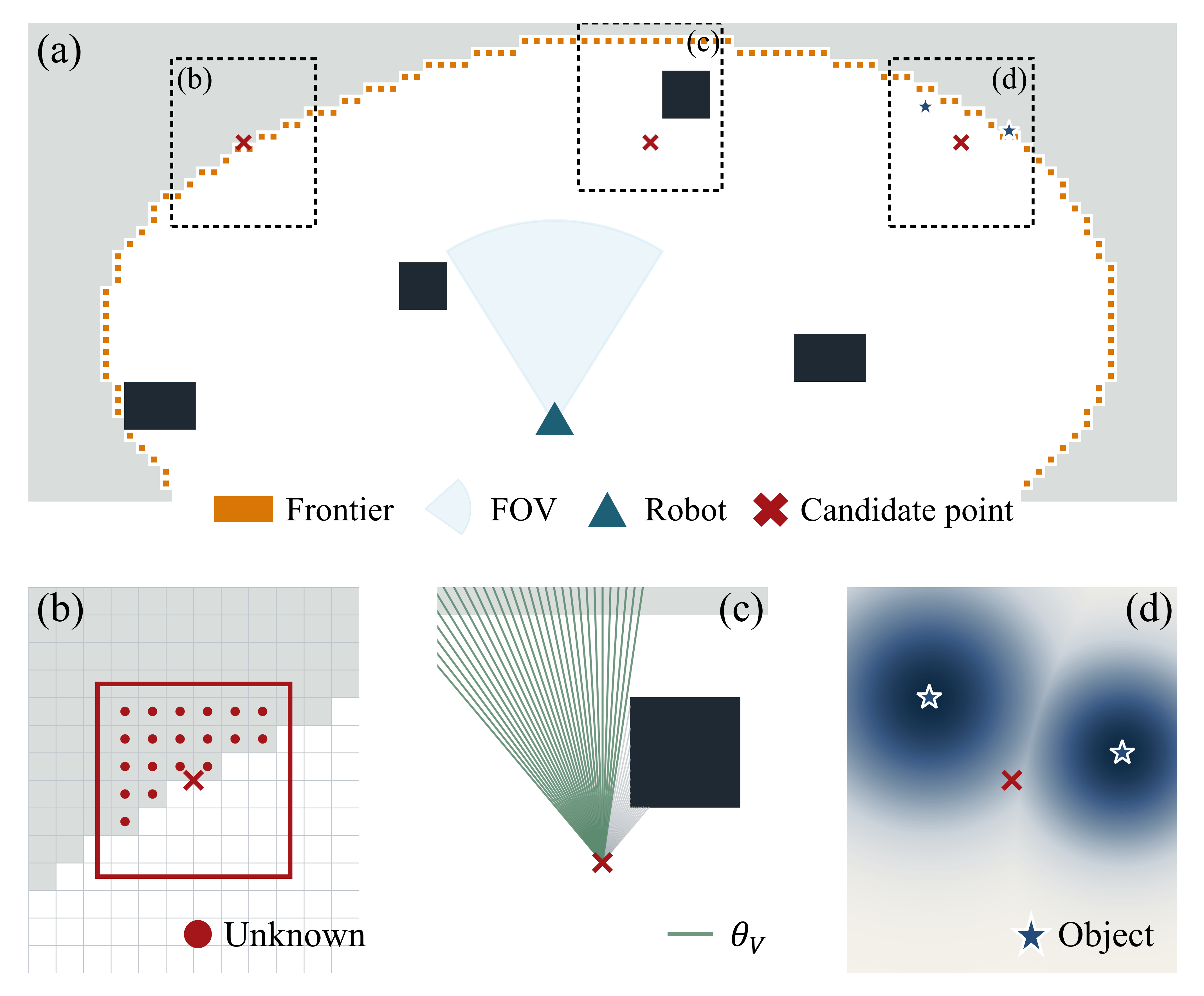}
    \caption{\blue Probability calculation for candidate points in the \abbrHATS-U. For each candidate point $x$, we compute (b) the density of unknown space $\phi_u(x)$, (c) the geometric visibility $\phi_g(x)$, and (d) the prior knowledge of object $\phi_o(x)$. These terms are linearly combined and utilized as the vertex probability within an \abbrTSPPT instance.
    }
    \label{tsppt:fig:probability_map}
\end{figure}

\abbrHATS-U then prioritizes these frontier cells by computing the ``information gain'' for each of them using a weighted average of three factors.
As depicted in Fig. \ref{tsppt:fig:probability_map}, these three factors are: the density of unknown space $\phi_u$, the geometric visibility $\phi_g$, and prior knowledge of object locations $\phi_o$ (o for object), and the final probability is:
\begin{align}
    p(x) = w_u\phi_u(x)+w_g\phi_g(x) + w_o\phi_o(x),
\end{align}
where $w_u,w_g,w_o$ are non-negative weighting factor so that $p(x)$ is within range $[0,1]$.
Specifically:
     
 The density of unknown space $\phi_u(x)$ quantifies the ratio of unknown cells to the total number of cells in a neighborhood around a frontier cell $x$. For example, we use a square neighborhood of fixed size centered on $x$ in our experiments (Fig.~\ref{tsppt:fig:probability_map}~(b)).
 
 For any cell $x$ in the occupancy grid, the geometric visibility $\phi_g(x)$ describes the proportion of unknown cells that are visible from $x$ after considering the occlusion by static obstacles.
 As Fig.~\ref{tsppt:fig:probability_map}~(c) shows, let $\theta_{FOV}$ denote the sensor's Field Of View (FOV), and let $\theta_{V}$ denote the total angle of the visible unknown space from $x$ (that are not blocked by obstacles), which can be computed by ray casting. Then, $\phi_g(x)$ is the ratio of $\theta_{V}$ over $\theta_{FOV}$, within a range between 0 and 1.
 As the robot moves, we compute the $\theta_{V}$ based on the robot's current pose.
 
 The prior knowledge $\phi_o$ is a potential field over the entire workspace, and the value $\phi_o(x)$ for any cell $x$ in the occupancy grid describes the likelihood for the object to be at $x$.
 As shown in Fig.~\ref{tsppt:fig:probability_map}~(d), we use a mixture of Gaussians to describe the prior knowledge of object locations.
 Let $X_o$ denote a set of possible object locations. In other words, for each location $x_k \in X_o$, $p_k(x) \sim \mathcal{N}(\mu_k,\sigma_k)$ denotes a Gaussian with mean value $\mu_k=x_k$ and a covariance matrix $\sigma_k$ describing the uncertainty, where $p_k(x)$ is the probability value of any cell $x$ in the workspace based on that Gaussian.
 Then, $\phi_o(x)=\max_{x_k \in X_o} p_k(x)$ is equal to the maximum probability value among these Gaussians.

\subsubsection{Goal Clustering}
\label{tsppt:subsec:goal_cluster}
{\blue 
During target search in an unknown environment, there can be many spatially adjacent frontier cells, and directly treating each frontier cell as a graph vertex would unnecessarily enlarge the resulting \abbrTSPPT instance. We therefore use a probability-weighted Mean Shift algorithm~\cite{comaniciu2002mean} to cluster frontier cells into a compact set of candidate search regions.
Compared with the standard Mean Shift, the modified algorithm will prioritize regions of high target probability.} 
Specifically, let $\hat{c}(x)$ denote the cluster center for each frontier cell $x$, with its initial value set to the coordinates of $x$ itself. Then, for a given neighborhood $N$ of $\hat{c}(x)$, a new cluster center $\hat{c}_{new}(x)$ is calculated as the weighted mean of the neighboring coordinates $x_i \in N$ $\hat{c}_{new}(x) = \frac{\sum_{i\in N} x_i \cdot \omega_i }{\sum_{i\in N} \omega_i}$, where the weighting coefficient $\omega_i$ accounts for both the probability value and spatial proximity of the frontier cell, {\blue given by $\omega_{i}=p(x_i) \cdot \exp \left(-\frac{\left|\hat{c}(x)-x_{i}\right|^{2}}{2 \delta^{2}}\right)$, and $\delta$ is the Mean Shift bandwidth.}

The clustering process continues iteratively until convergence, which is determined by a stopping condition. Specifically, the update of the cluster center $\hat{c}_{new}(x)$ is considered complete when the change between successive cluster centers is below a predefined threshold $\delta_c$, {\blue satisfying $|\hat{c}_{new}(x)-\hat{c}(x)|<\delta_{c}.$}
After convergence, a post-processing merging step is applied. It merges cluster centers that are within a defined distance threshold $\delta_{dist}$, which reduces the final dimension of the graph $G$ for the \abbrTSPPT instance.

%% file: result.tex
Our experiments include evaluation of both algorithms for \abbrTSPPT and systems for target search.
{\blue \begin{itemize}
    \item To evaluate our algorithms (Sec.~\ref{TSPPT:sec:method}), we tested the proposed \abbrRPT algorithm against two exact baselines: the integer programming (IP)~\cite{Behdani2014AnIA} and Held-Karp dynamic programming (DP)~\cite{held1962dynamic}, both adapted to \abbrTSPPT and detailed in Appendix~\ref{TSPPT:app:exact}. We also conducted an ablation study of \abbrRPT to understand the impact of heuristic and focal search. Finally, we compared our approach with several baseline methods in terms of the runtime and solution quality.
    \item To evaluate our \abbrHATS systems (Sec.~\ref{TSPPT:sec:system}), we compared \abbrHATS-L and \abbrHATS-U against several baselines in both Gazebo simulations and real robot experiments.
\end{itemize}}
As shown in Fig.~\ref{tsppt:fig:example}, we used two types of datasets including (i) a public dataset of city-like networks from TSPLIB \cite{reinelt1991tsplib}, a popular library for TSP-related problems, and (ii) a synthetic dataset.
The synthetic dataset consists of graphs with varying number of vertices that are randomly sampled from a 2D plane {\blue in two separate modes. For experiments whose graph sizes do not exceed $40$, vertices are sampled from a $500 \times 500$ region, and for experiments whose graph-size range extends beyond $40$, vertices are sampled from a larger $5000 \times 5000$ region for all tested instances.
In all cases, the distance between any two vertices must exceed $5$. Unless otherwise stated, for each fixed number of vertices $N$, $50$ instances are generated.}
All methods were implemented in C++ and run on a laptop with an Intel Core i7-10875H 2.30 GHz CPU and 16 GB of RAM.

\begin{figure}[tb]
    \centering
    \includegraphics[width=\linewidth]{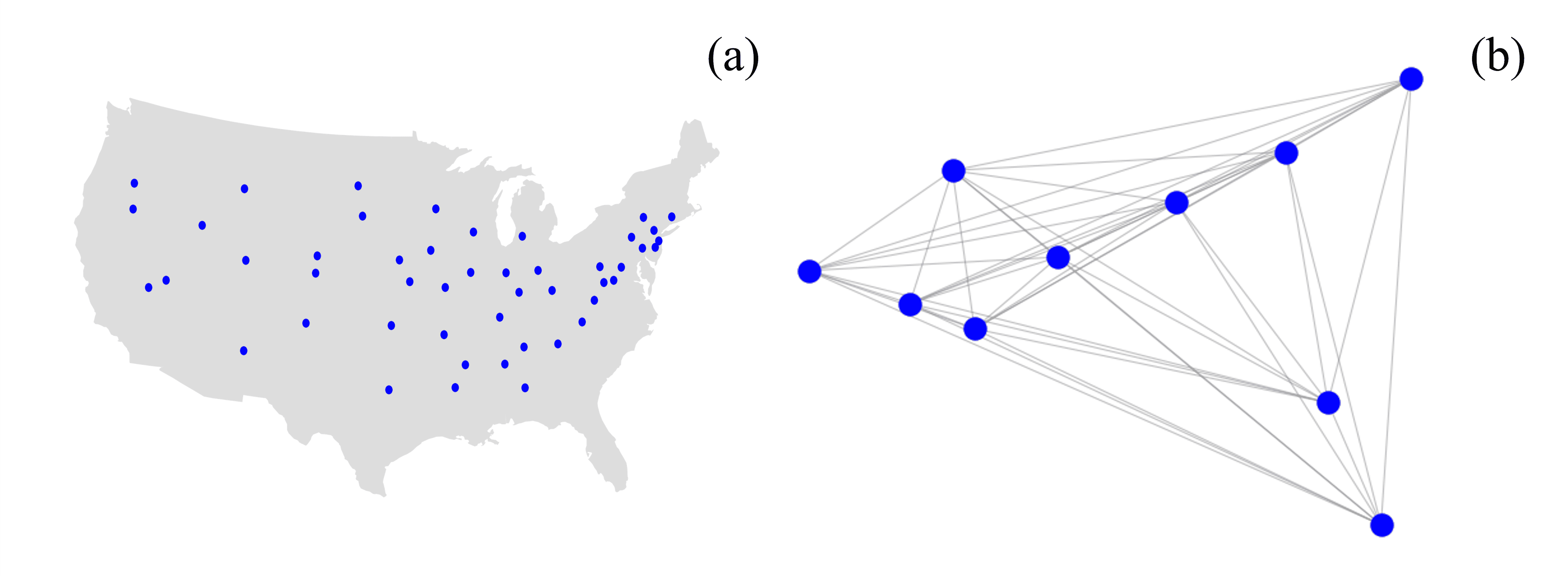}
    \vspace{-5mm}
    \caption{Two types of datasets for experiment. (a) is the att48 instance from public TSPLIB dataset: 48 U.S. capital cities. (b) is an example of the synthetic dataset: a randomly generated graph with 10 vertices.}
    \label{tsppt:fig:example}
\end{figure}

{\blue \subsection{Comparison Against Exact Baselines}
This subsection compares \abbrRPT with two exact baselines for \abbrTSPPT: IP and DP, as elaborated in Appendix~\ref{TSPPT:app:exact}.
}

\subsubsection{IP Baseline}
{\blue The IP baseline can also take advantage of the proposed transformation method in Eq.~\eqref{tsppt:lem:xi_rewrite}.
We denote IP and IP$_{tf}$ as the naive IP formulation and the IP formulation that leverages the transformed formulation, respectively.}
We tested on five TSPLIB instances~\cite{reinelt1991tsplib} with different graph sizes: gr17, gr21, gr24, fri26, and bays29, where the numbers of vertices are $17$, $21$, $24$, $26$, and $29$, respectively.
Gurobi~\cite{gurobi} is used as the solver for both IP formulations.

\begin{figure}
    \centering
    \includegraphics[width=0.96\linewidth]{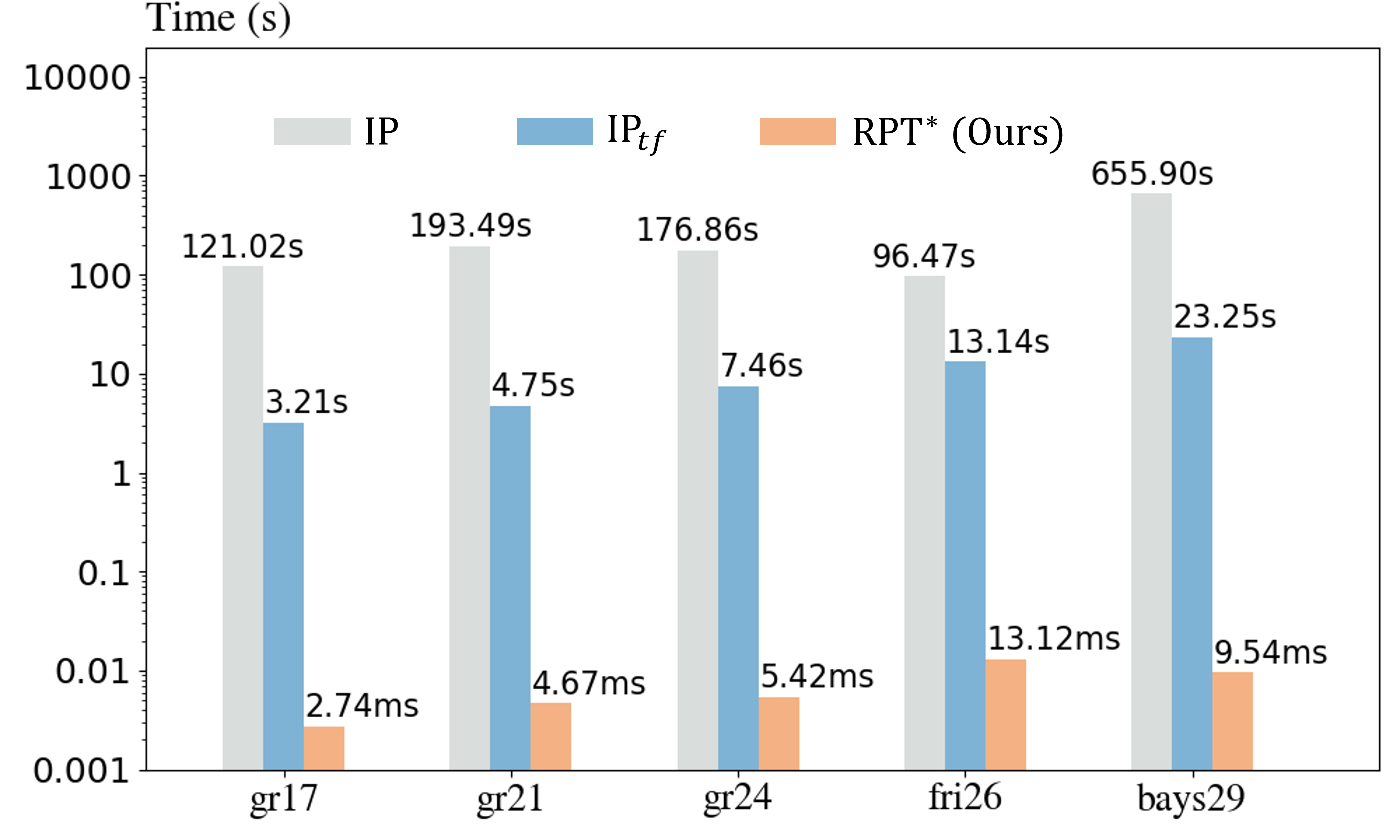}
    \caption{Average runtime comparison of IP baselines and our \abbrRPT.}\label{tsppt:fig:IP_time}
\end{figure}

As shown in Fig.~\ref{tsppt:fig:IP_time}, IP requires hundreds of seconds, while IP$_{tf}$ requires a few seconds, which is obviously faster.
Furthermore, our \abbrRPT requires several milliseconds, about three orders of magnitude faster than the transformed IP formulation.

Besides, we observe that both IP formulations often suffer from numerical errors of floating numbers and result in errors in the returned optimal solutions.
As shown in Table~\ref{tsppt:tab:IP_cost}, the IP formulations often find slightly more expensive solutions than \abbrRPT although the solver claims to solve the problem to optimality (when we configured the optimality gap to be zero).
Recall that, in \abbrTSPPT, due to the accumulated probability and products of floating numbers, such small errors can lead to big differences in the resulting solution paths.
For this reason, the IP formulations are used only as exact-optimization baselines for small instances and are omitted from the remaining comparisons.

\begin{table}[tb]
    \centering
    \caption{IP vs. \abbrRPT solution cost comparison}
    \label{tsppt:tab:IP_cost}
    \begin{tabular}{l|c|c|c}
        \hline
        \textbf{Instance} & \textbf{IP Cost} & \textbf{\abbrRPT Cost} & \textbf{\abbrRPT Savings} (vs. IP) \\
        \hline
        gr17 & 16.1545 & 16.1473 & 0.045\% \\
        \hline
        gr21 & 13.9029 & 13.9024 & 0.004\% \\
        \hline
        gr24 & 9.4408 & 9.4400 & 0.008\% \\
        \hline
        fri26 & 13.2027 & 13.1946 & 0.061\% \\
        \hline
        bays29 & 12.9606 & 12.9526 & 0.062\% \\
        \hline
    \end{tabular}
\end{table}

{\red \subsubsection{DP Baseline}
After the expected cost is transformed into the additive form in Sec.~\ref{TSPPT:sec:method}, the application of Held-Karp dynamic programming to the \abbrTSPPT becomes possible. 
The DP baseline takes advantage of the transformation in Eq.~\eqref{tsppt:lem:xi_rewrite} since otherwise, the state space involves all permutations of vertices which is prohibitively large.
The recurrence of this DP baseline is detailed in Appendix~\ref{TSPPT:app:dp}. 
We ran tests on our randomly generated dataset. }

\begin{figure}[tb]
    \centering
    \includegraphics[width=\linewidth]{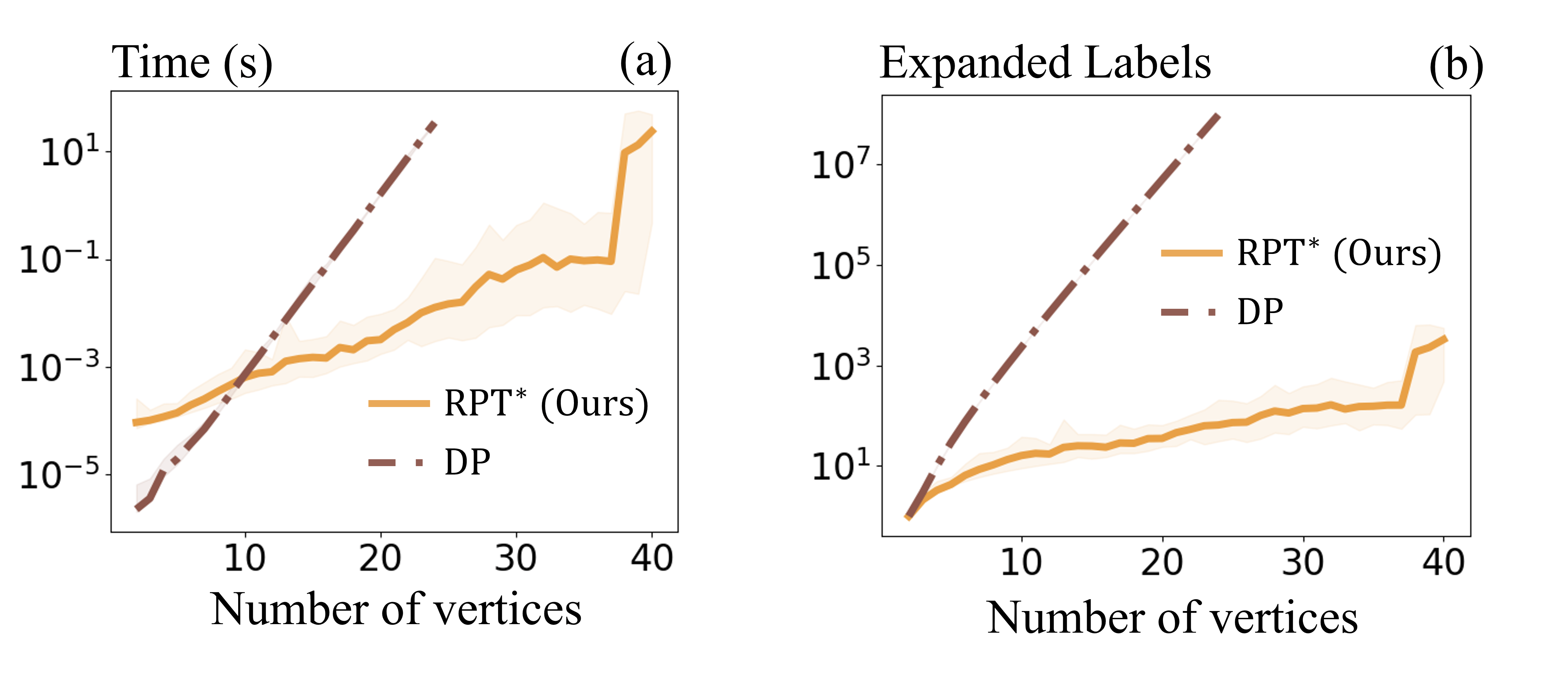}
    \vspace{-5mm}
    \caption{\blue Comparison of the DP baseline with \abbrRPT. (a) Runtime comparison. DP solves instances with up to approximately $24$ vertices within the runtime limit, whereas \abbrRPT solves instances with up to approximately $40$ vertices to optimality. (b) Number of processed states. The DP baseline computes an order of magnitude more states than \abbrRPT, and the gap increases rapidly as the number of vertices grows.}
    \label{tsppt:fig:dp}
\end{figure}

{\red As Fig.~\ref{tsppt:fig:dp} shows, while \abbrRPT can solve instances with up to approximately $40$ vertices to optimality, the DP baseline solves instances with up to only approximately $24$ vertices within the $60$-second runtime limit. Furthermore, to achieve these solutions, the DP baseline processed more than an order of magnitude more states than \abbrRPT for the same graph size, and as the number of vertices grows, the gap increases rapidly.
These results confirm the computational advantage of the proposed \abbrRPT over DP.}

\subsection{Ablation Study of \abbrRPT and \abbrBRPT}
This section presents the impact of the heuristic and focal search on \abbrRPT, evaluated on the randomly generated dataset. 

\subsubsection{Impact of Heuristic}
We first set the $h$-value to $0$ (denoted as $\text{\abbrRPT}_{noh}$) which is used as an uninformed baseline, and then compared its performance against the full algorithm (\abbrRPT) with the heuristic as described in Sec.~\ref{tsppt:subsec:heuristic}.

\begin{figure}
    \centering
    \includegraphics[width=\linewidth]{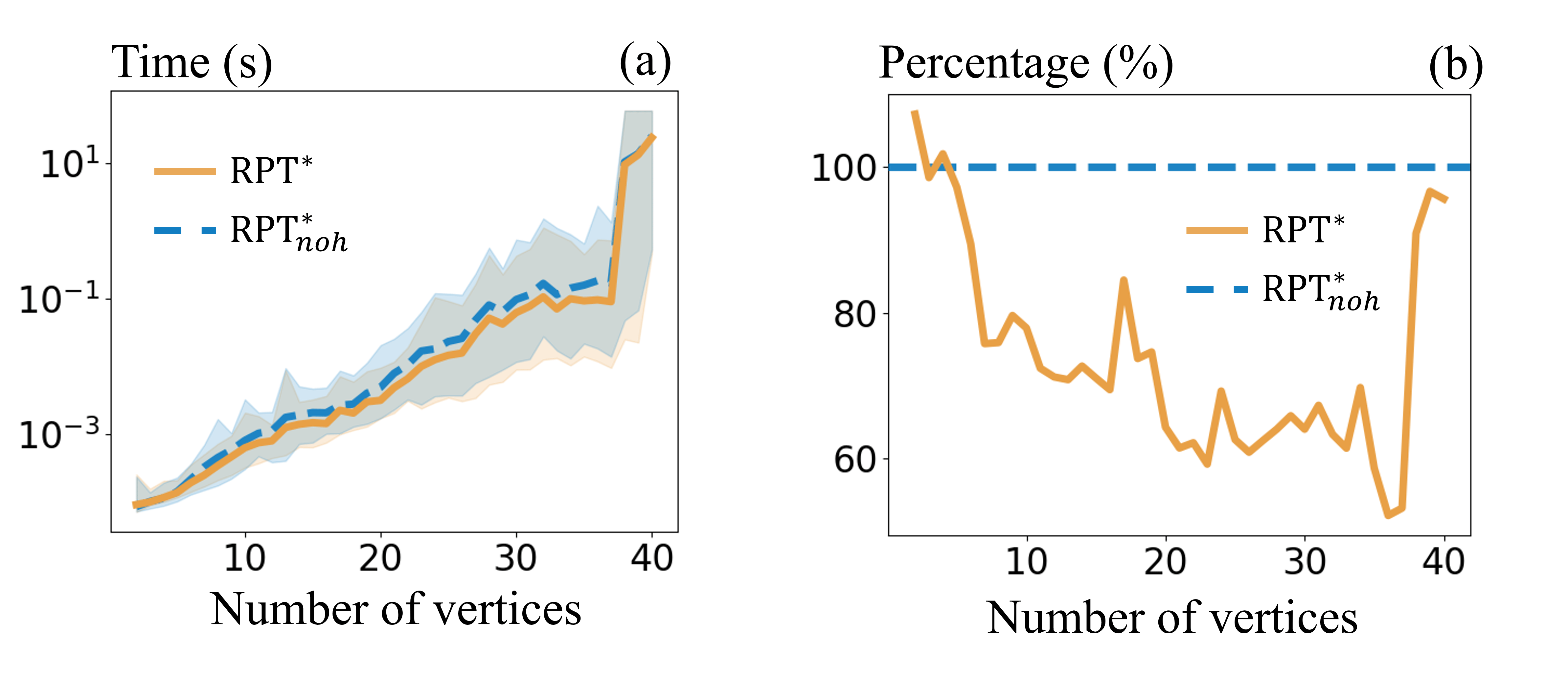}
    \vspace{-5mm}
    \caption{Impact of heuristic. (a) Runtime comparison. {\blue (b) Normalized runtime ratio. For graphs with up to $40$ vertices, both methods have almost the same success rates, while heuristic helps save $30\%$-$50\%$ runtime than the uninformed \abbrRPT.}}
    \label{tsppt:fig:heuristic}
\end{figure}

The results show that, for very small graphs, \abbrRPT with the heuristic can be slightly slower because of the $O(|V|^2)$ preprocessing step for the heuristic table and the additional heuristic evaluation overhead during search. As the graph size increases, the reduction in search effort becomes more evident, and the normalized runtime ratio is generally below $100\%$. For nontrivial graph sizes before timeout effects dominate, \abbrRPT often requires approximately $50\%$--$70\%$ of the runtime of $\text{\abbrRPT}_{\mathrm{noh}}$, corresponding to an empirical runtime reduction of roughly $30\%$--$50\%$. When $N$ exceeds approximately $37$, the ratio increases because the success rate of \abbrRPT starts to decrease and the average runtime becomes dominated by runs close to the $60$~s time limit. Therefore, this increase near the largest tested graph sizes should be interpreted as a timeout effect rather than as evidence that the heuristic becomes ineffective. Overall, the results indicate that the heuristic is useful for reducing search effort on nontrivial \abbrTSPPT instances where both methods terminate reliably.
{\blue Fig.~\ref{tsppt:fig:heuristic} shows the runtime and the normalized runtime ratio of \abbrRPT over $\text{\abbrRPT}_{noh}$ respectively, where the values below $100\%$ indicate that \abbrRPT with the heuristic is faster.

The result shows that, for very small graphs, \abbrRPT with the heuristic can be slightly slower because of the $O(|V|^2)$ pre-processing step for the heuristic table. As the size of graph increases, the reduction in search effort becomes more clear, and the normalized runtime ratio is generally below $100\%$. For the nontrivial instance, \abbrRPT often requires approximately $50\%$-$70\%$ of the runtime of $\text{\abbrRPT}_{noh}$, corresponding to a runtime reduction of roughly $30\%$--$50\%$, which verifies the advantage of the proposed heuristic.
When $N>38$, the ratio increases because the success rate of \abbrRPT starts to decrease and the average runtime of both planners becomes close to the $60$~s time limit.}

\subsubsection{Runtime of \abbrBRPT}
\abbrRPT solves the problem to optimality, but can only handle graphs with around 40 vertices as we observed in the results.
In contrast, \abbrBRPT can handle much larger graphs.
As shown in Fig.~\ref{tsppt:fig:focal}~(a), we assessed \abbrBRPT on graph sizes from $10$ to $200$ vertices with a step size of $10$, with $10$ instances generated for each size.
The runtime of \abbrBRPT slightly increases as the number of vertices increases.
To summarize, \abbrBRPT can handle graphs with up to about $200$ vertices within a runtime limit of $60$ seconds.

\subsubsection{Varying $\epsilon$ of \abbrBRPT}

\begin{figure}
    \centering
    \includegraphics[width = \linewidth]{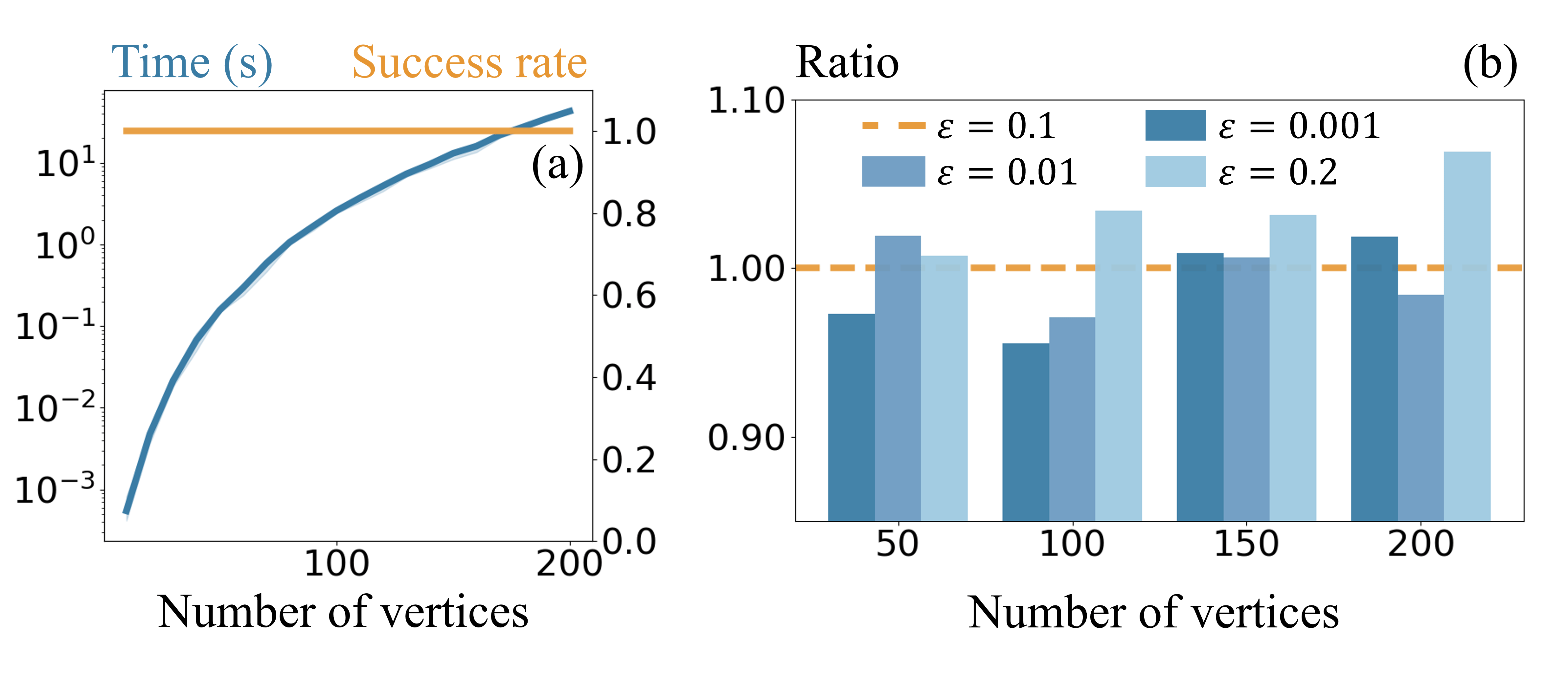}
    \vspace{-5mm}
    \caption{Impact of focal search. (a) Runtime and success rates of \abbrBRPT ($\epsilon = 0.01$). (b) Comparison of different $\epsilon$. Focal search helps \abbrRPT to solve problems with up to $200$ vertices. A small $\epsilon$ can speed up the search and further increasing $\epsilon$ does not provide much additional reduction in the runtime.}
    \label{tsppt:fig:focal}
\end{figure}

We then tested \abbrBRPT with varying $\epsilon \in \{0.001,0.01,0.1,0.2\}$.
As shown in Fig.~\ref{tsppt:fig:focal}~(b), varying $\epsilon$ does not affect the runtime very much.
The reason is that even if using a small epsilon (such as $\epsilon=0.01$), there can be many states that enter the focal list, which can speed up the search. 

{\blue \subsection{Comparison Against Non-optimal Baselines}
This subsection compares \abbrRPT and \abbrBRPT with non-optimal baselines. We first introduce the baselines, and then present the comparison results.

\subsubsection{Baselines}
The \textbf{first} baseline Greedy on Probability (GP), uses a greedy strategy to solve \abbrTSPPT by letting the robot always move to a vertex $v$ in the graph $G$ with the largest probability value $p(v)$ among the unvisited vertices.
The \textbf{second} baseline Greedy on Distance (GD), always selects an unvisited vertex $v$ in the graph with the nearest edge cost from the current vertex.
The \textbf{third} baseline ignores the probability by treating the probability of all vertices simply as one.
As a result, \abbrTSPPT becomes a regular HPP and we use a popular solver Lin-Kernighan-Helsgaun (LKH)~\cite{helsgaun2000effective} to solve the resulting HPP.\footnote{LKH solves a TSP problem. An HPP can be converted to a TSP by making a copy $v_s'$ of the start vertex $v_s$ and adding a zero-cost edge between $v_s'$ and all other vertices. After obtaining the solution tour from LKH, one can remove $v_s'$ to obtain a solution to HPP.} We refer to this baseline as LKH.
The \textbf{fourth} baseline utilizes the RL4CO framework~\cite{berto2023rl4co}, a recent learning-based combinatorial optimization method, where each vertex is encoded by its coordinates and probability of finding the target, and the method returns a visiting order permutation. We refer to this as RL4CO.

\subsubsection{Runtime and Solution Costs}
All methods, except for RL4CO, are tested in the synthetic dataset with up to 40 vertices.
Fig.~\ref{tsppt:fig:comparison}~(a) reports the runtime on a logarithmic scale, while Fig.~\ref{tsppt:fig:comparison}~(b) compares the expected solution costs based on Eq.~\eqref{tsppt:eq:expected_cost}.}

As depicted in Fig.~\ref{tsppt:fig:comparison}, while GP runs orders of magnitude faster than our \abbrRPT, its solution cost often doubles or triples the optimal cost returned by \abbrRPT (Fig.~\ref{tsppt:fig:comparison}~(b)).
Similarly, GD and LKH are faster than \abbrRPT but find solutions that are more expensive than the optimal solution found by \abbrRPT.
Overall, the results show a clear trade-off: non-optimal baselines are faster but produce higher expected costs, whereas \abbrRPT provides optimal solutions and \abbrBRPT offers a practical compromise between runtime and solution quality.
The next subsection evaluates how such a gain in solution quality affects the robot's performance in target search.

\begin{figure*}[tb]
    \centering
    \includegraphics[width=0.9\textwidth]{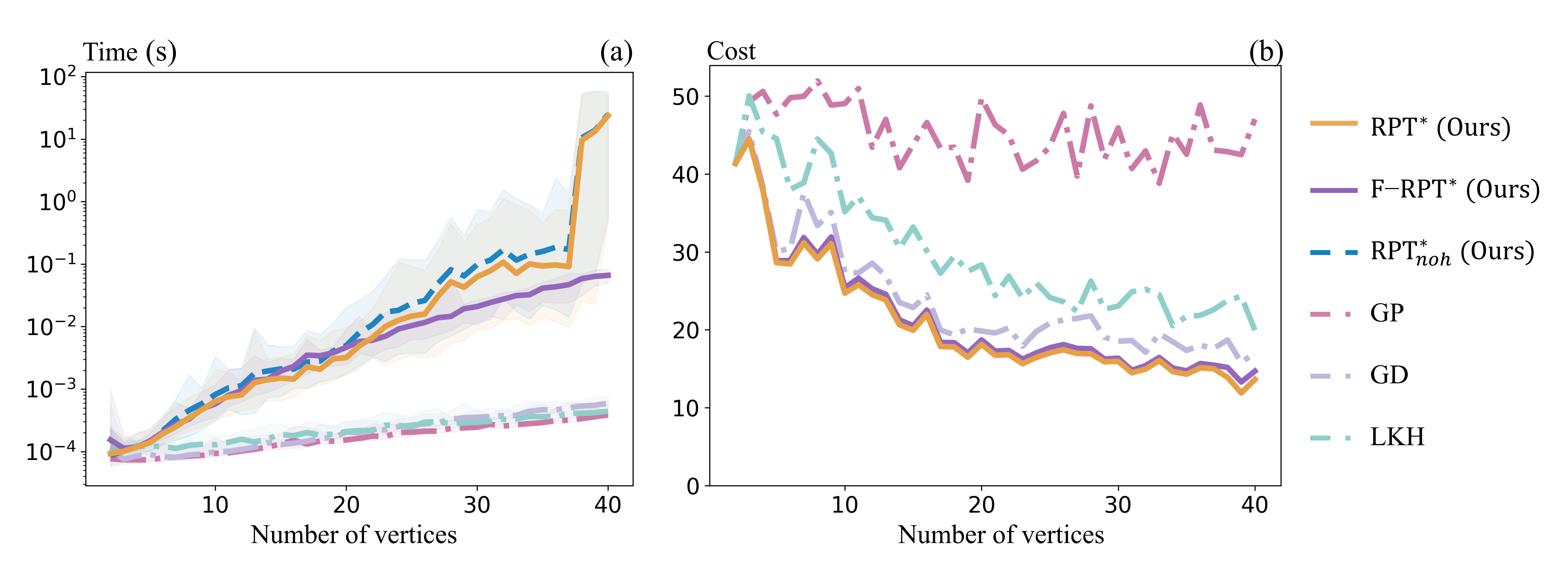}
    \vspace{-5mm}
    \caption{Comparison against baselines. (a) Runtime comparison. (b) Solution cost comparison. In both (a) and (b), {\blue solid lines denote the proposed methods, \abbrRPT and \abbrBRPT; the dashed line denotes the uninformed ablation baseline $\text{\abbrRPT}_{\mathrm{noh}}$; and dash-dotted lines denote the non-optimal baselines GP, GD, and LKH.} Our \abbrRPT and \abbrBRPT run slower than the baseline but find better quality solutions.}
    \label{tsppt:fig:comparison}
\end{figure*}

\subsubsection{Comparison Against RL4CO}
In this RL4CO baseline~\cite{berto2023rl4co}, the policy network uses an Attention Model~\cite{kool2018attention} to output a visiting permutation from graph features, and each vertex is represented by geometric coordinates and probability of finding the target. The reward is defined as $R(\pi)=-\xi(\pi)$, where $\xi(\pi)$ is the \abbrTSPPT expected cost in Eq.~\eqref{tsppt:eq:expected_cost} so that maximizing the reward encourages minimizing the expected search cost during the training process. 

The model is trained offline on synthetic instances of $N=40$ vertices, where the coordinates and the corresponding probabilities are randomly generated. We conducted training on an NVIDIA RTX 4080 GPU with mixed precision using $100{,}000$ training instances and $10{,}000$ validation instances, with a batch size of $512$, a learning rate of $10^{-4}$, a rollout baseline, and $100$ training epochs.
During evaluation, a path is built by iteratively selecting the best next vertex as predicted by the network (i.e., greedy decoding).

We compared RL4CO with \abbrBRPT on the synthetic dataset with up to $200$ vertices and the reported runtime of RL4CO includes only online inference and decoding, excluding offline training.
As shown in Fig.~\ref{tsppt:fig:rl4co}, for small graphs, RL4CO obtains expected costs comparable to those of \abbrBRPT, but the fixed neural-network inference overhead makes it slower than \abbrBRPT.
As the graph size increases beyond $N=40$, the runtime of RL4CO grows more slowly than that of \abbrBRPT, and outperforms \abbrBRPT.
{\red However, the expected cost obtained by RL4CO becomes noticeably higher than that of \abbrBRPT, up to approximately 30\% higher than the cost of \abbrBRPT at around $N=200$.
This result indicates that the learned policy has limited generalization capability to unseen scenarios, when evaluated on graphs substantially larger than those used for training.
In contrast, \abbrBRPT finds better quality solutions for graphs with up to $200$ vertices.}

\begin{figure}[tb]
    \centering
    \includegraphics[width=\linewidth]{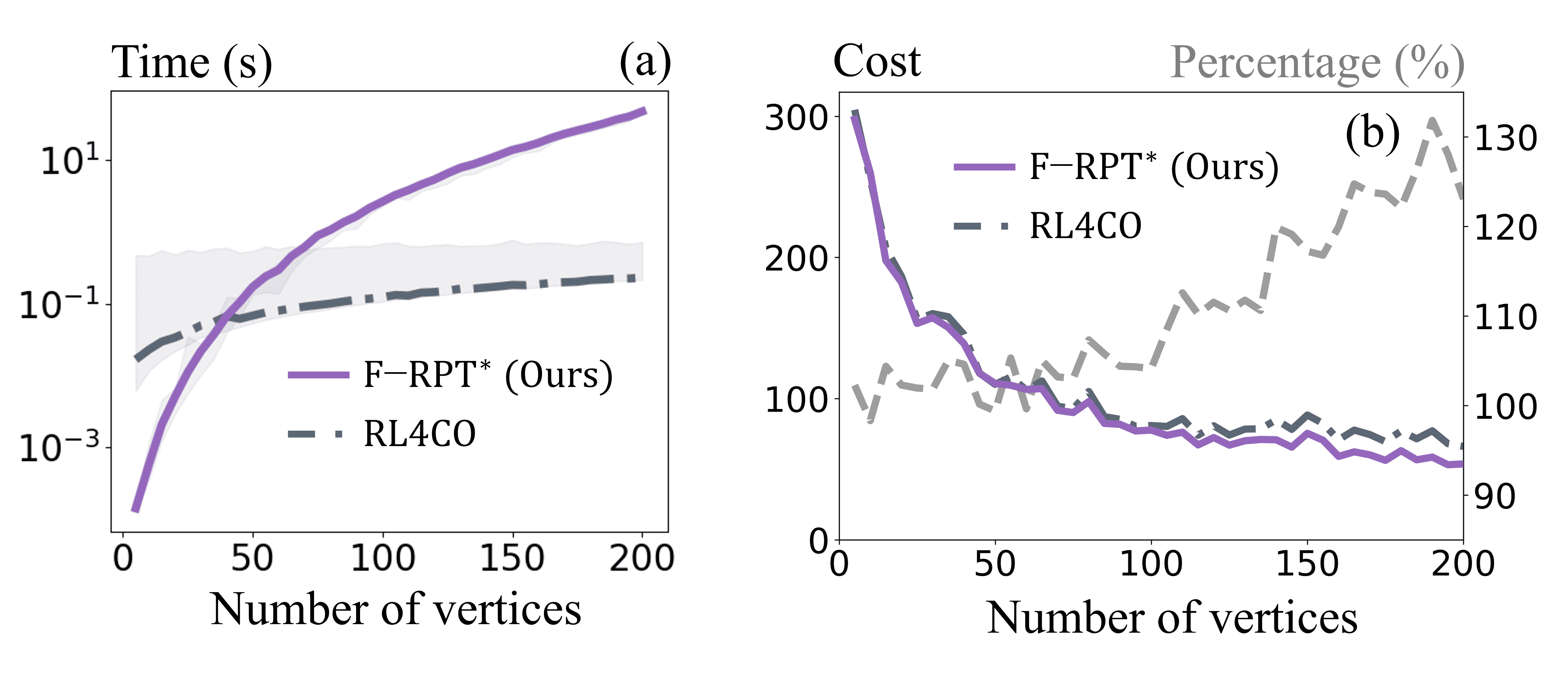}
    \vspace{-5mm}
    \caption{\blue Comparison between \abbrBRPT and the RL4CO-based learning baseline. (a) Runtime comparison. (b) Solution cost comparison and normalized ratio of RL4CO with respect to \abbrBRPT. RL4CO provides fast online inference and competitive solutions, while \abbrBRPT retains the bounded-suboptimality guarantee and achieves lower expected cost in the tested instances.}
    \label{tsppt:fig:rl4co}
\end{figure}

\subsection{Lifelong Target Search with Bayesian Update}
\label{tsppt:subsec:HATS-L}
\begin{figure}[tb]
    \centering
    \includegraphics[width=\linewidth]{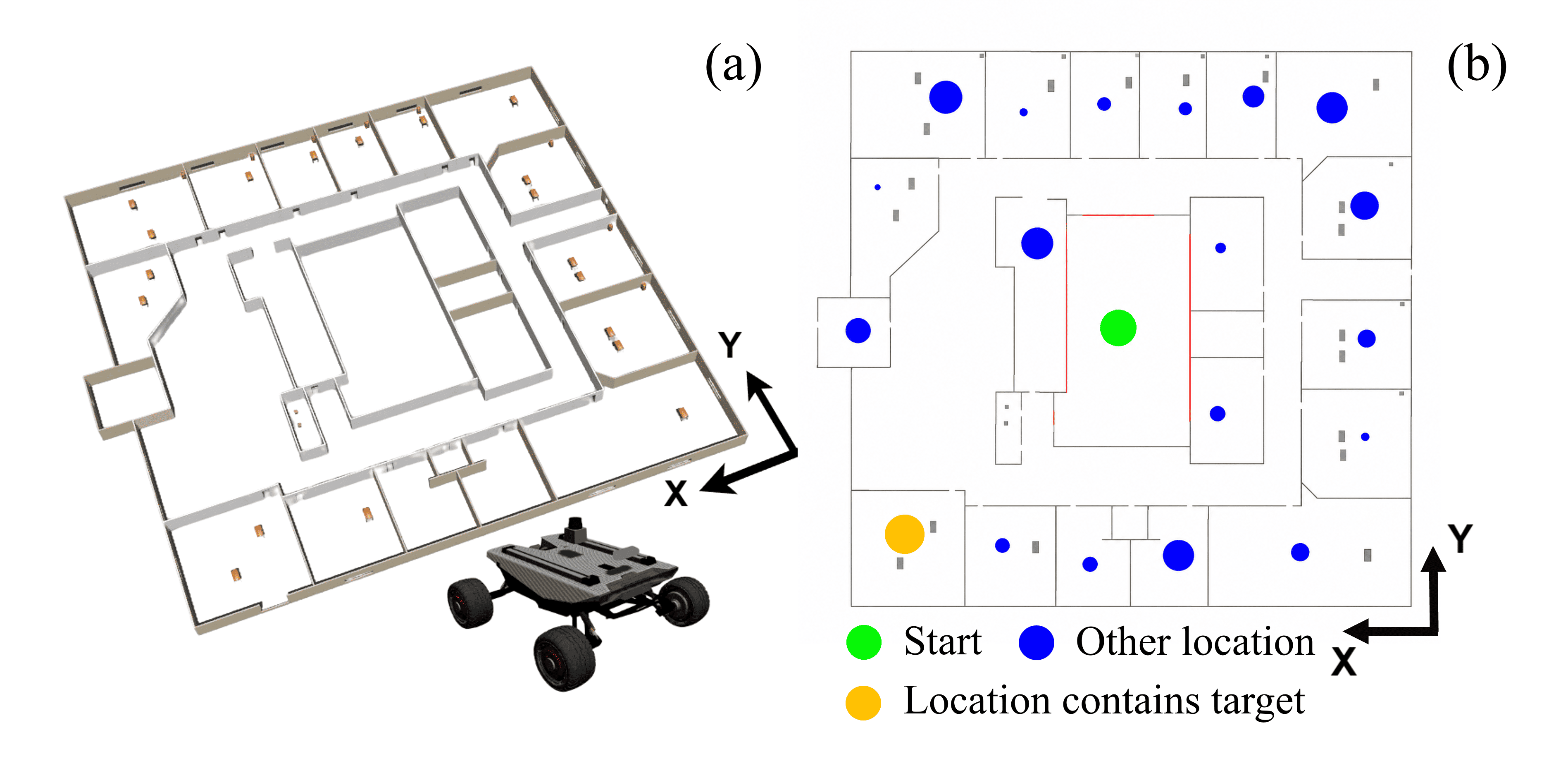}
    \vspace{-5mm}
    \caption{Experimental setup for lifelong target search. (a) The $80 \times 80$~m simulation environment and robot model in Gazebo. {\blue (b)~Pre-constructed map with $20$ candidate target locations. In each trial, one location is randomly selected as the true target location and is marked in orange, while the remaining locations are target-absent and are marked in blue.} The size of the blue or orange circle indicates the sampled probability values at each location. A larger probability corresponds to a larger circle.}
    \label{tsppt:fig:HATS-L}
\end{figure}

We evaluated \abbrHATS-L in an $80 \times 80$~m Gazebo environment with $20$ candidate target locations.
In each trial, one of the locations is selected as the true target location at random, and the remaining locations are treated as target-absent. The initial belief distribution over all candidate locations is randomly sampled once per trial and kept identical for all compared planners.
We pre-constructed a graph where edge weights are set as the path length returned by A* search between each pair of target locations in an occupancy grid.
We used the parameters $\alpha_{1} = 0.8, \alpha_{2} = 0.4$ for the sensor observation model as described in Sec.~\ref{TSPPT:sec:HATS-L}.
We used the following termination condition for this ``lifelong'' test: a location is confirmed as ``target-present'' if its probability exceeds $p_{H}=0.98$, or ``target-absent'' if it falls below $p_{L}=0.02$.
Upon reaching either threshold, the vertex is removed from the search graph, and the search process terminates when the robot finds the target.
We use ``mission duration'' to denote the total runtime until the robot stops. 
To isolate the effect of the Global Planning, we replaced only the global solver with the aforementioned non-optimal baselines while keeping other modules identical across all methods. 

As shown in Fig.~\ref{tsppt:fig:Result-HATS-L}, \abbrRPT achieves the lowest median mission duration with small variance.
GP can visit high-probability locations early, but it may incur unnecessary travel because it does not account for travel distance.
GD and LKH reduce travel distance, but they ignore the probability of finding the target and therefore can delay the visits to locations with high probability.
RL4CO achieves better mission time than GD and LKH but suffers from large variance in this randomized setting.
These results show that \abbrRPT achieves better mission time for lifelong target search with Bayesian updates due to the optimization of the expected cost in Eq.~\eqref{tsppt:eq:expected_cost}.

\begin{figure}[tb]
    \centering
    \includegraphics[width=0.96\linewidth]{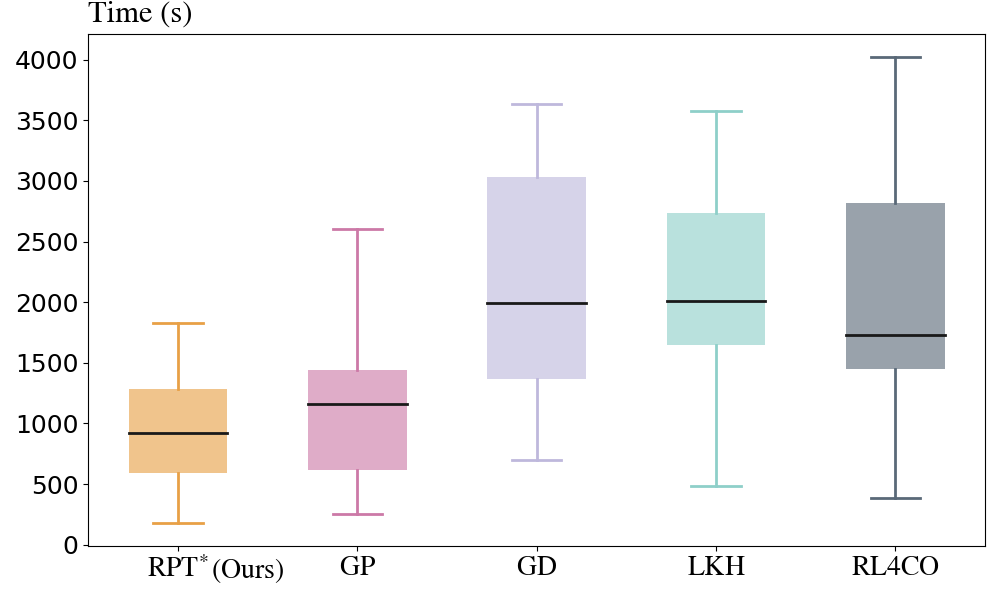}
    \caption{Mission duration of \abbrHATS-L with different global solvers. The boxplot reports the mission duration of \abbrRPT, GP, GD, LKH, and the RL4CO-based baseline over $20$ randomized trials. Each box shows the interquartile range, the center line shows the median, and the whiskers indicate the empirical minimum and maximum. \abbrRPT achieves the lowest median mission duration and a relatively compact distribution, demonstrating the benefit of optimizing the \abbrTSPPT expected cost under Bayesian belief updates.}
    \label{tsppt:fig:Result-HATS-L}
\end{figure}

\subsection{Simulation of Target Search in Unknown Environments}

\begin{figure}[tb]
    \centering
    \includegraphics[width=\linewidth]{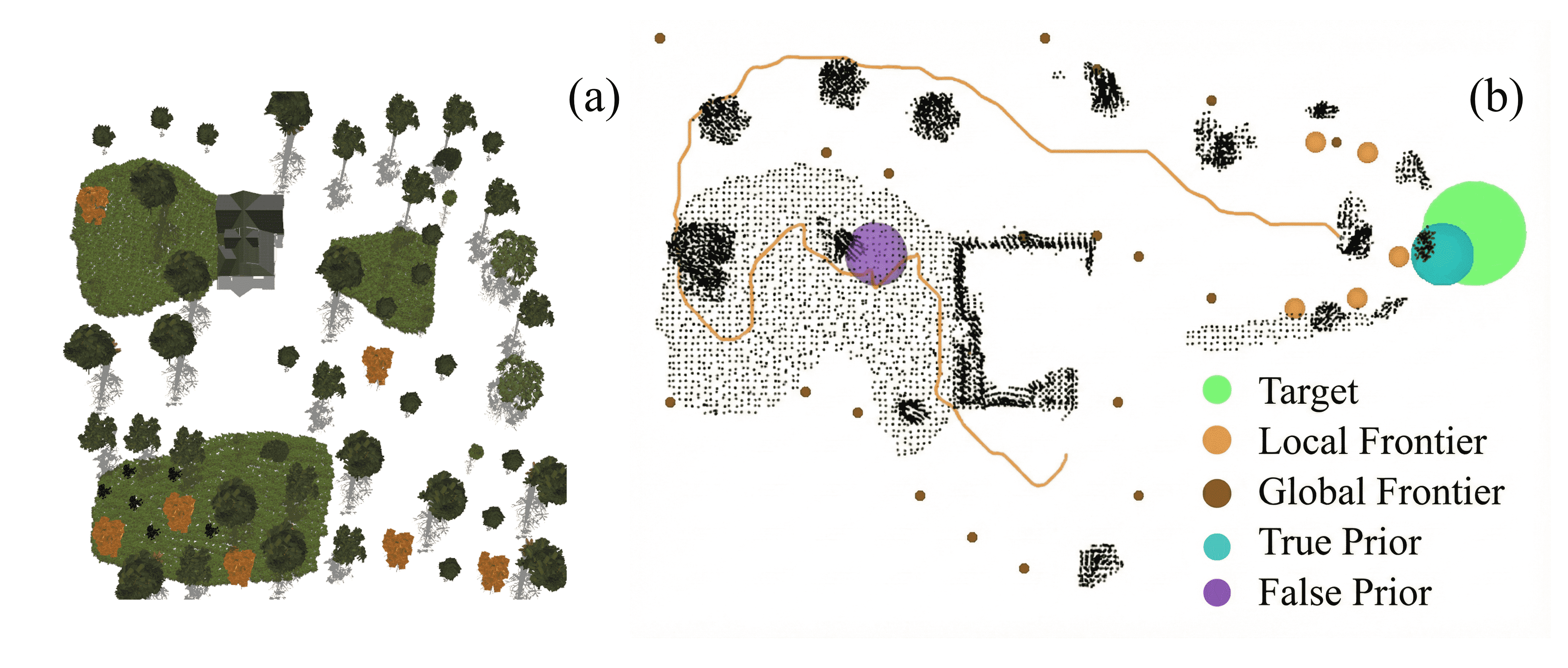}
    \vspace{-5mm}
    \caption{Experimental setup for target search in unknown environments. (a) The $101 \times 101$~m forest environment. (b)~Visualization of the search process.}
    \label{tsppt:fig:HATS-U}
\end{figure}

\subsubsection{Simulation Setup}
As shown in Fig.~\ref{tsppt:fig:HATS-U}~(a), we first evaluated the HATS-U system within a simulated unknown forest environment of size $101 \times 101$~m, where the robot autonomously explores and searches for the target following the procedures detailed in Sec.~\ref{TSPPT:sec:HATS-U}.

In this experiment, there are 20 candidate target locations, and for each trial we consider only one target. The prior knowledge of its location is modeled as a Gaussian-based distribution via $\phi_o$, as described in Sec.~\ref{tsppt:subsec:goal_cluster}.

The exploration process terminates when the distance between the robot and the target is less than a threshold of $5$~m, and we record the time as ``mission duration''. Otherwise, the mission is terminated at a time limit of $600$~s and is recorded as a timeout.
We replaced our \abbrRPT, GP, GD, LKH and RL4CO for the planner in Global Planning module, and evaluated over $10$ randomized trials. All other settings remain identical across all trials. 

\subsubsection{Accurate Prior Knowledge}
We first tested a scenario with accurate prior knowledge, where $\phi_o$ assigns a weight of $0.4$ to the true target, and distributes the remaining $0.6$ weight uniformly over the other candidate locations.

\begin{table}[t]
    \centering
    \caption{Accurate-prior results over $10$ randomized trials.}
    \label{tsppt:tab:accurate}
    \begin{tabular}{l|c|c}
        \hline
        \textbf{Planner} & \textbf{Success ($/10$)} & \textbf{Mission Duration (s)} \\
        \hline
        \abbrRPT\ (Ours) & $10$ & $77.2$ \\
        \hline
        GP               & $10$ & $\mathbf{49.5}$ \\
        \hline
        GD               & $10$ & $170.6$ \\
        \hline
        LKH              & $5$  & $439.2$ \\
        \hline
        RL4CO            & $3$  & $420.6$ \\
        \hline
    \end{tabular}
\end{table}

As Table~\ref{tsppt:tab:accurate} shows, GP obtains the shortest mission duration in the accurate-prior scenario. This result is expected, because when the prior knowledge matches the true target location, directly moving toward the high-probability region can lead to shortest mission time.
In this scenario, \abbrRPT achieves the second-shortest mission duration and, together with GP, substantially outperforms the remaining baselines.

\subsubsection{Misleading Prior Knowledge}
We then designed a second scenario with misleading prior knowledge, where $\phi_o$ assigns a weight of $0.4$ to both the true target and an additional decoy location that is far away from the true target, and the remaining $0.2$ weight is uniformly distributed over the other candidates.

\begin{table}[t]
    \centering
    \caption{Misleading-prior results over $10$ randomized trials.}
    \label{tsppt:tab:misleading}
    \begin{tabular}{l|c|c}
        \hline
        \textbf{Planner} & \textbf{Success ($/10$)} & \textbf{Mission Duration (s)} \\
        \hline
        \abbrRPT\ (Ours) & $\mathbf{10}$ & $\mathbf{90.7}$ \\
        \hline
        GP               & $9$  & $101.4$ \\
        \hline
        GD               & $9$  & $180.6$ \\
        \hline
        LKH              & $6$  & $484.8$ \\
        \hline
        RL4CO            & $2$  & $414.0$ \\
        \hline
    \end{tabular}
\end{table}

As Table~\ref{tsppt:tab:misleading} shows, in the misleading-prior scenario, GP is attracted to the decoy location, and therefore increases the mission duration and decreases the success rate.
In contrast, \abbrRPT shows only mild degradation, with a slight increase of mission duration, and is the \emph{only} planner that succeeds in all $10$ trials.
Therefore, under misleading prior information, \abbrRPT shows the highest reliability.

These results indicate that our \abbrRPT can balance exploitation (moving toward regions with high target probability as GP does), with exploration (maintaining coverage over alternative candidate regions such as LKH does).
This balance allows \abbrRPT to locate the target more quickly when the prior knowledge can be misleading.

\subsubsection{Hyperparameter Sensitivity Study}
\label{tsppt:subsec:sensitivity}
Our \abbrRPT finds an optimal solution within a given graph $G$ and has no hyper-parameter.
We therefore investigate the two hyper-parameters that affect the generation of $G$: the bandwidth $\delta$ and the post-processing merge threshold $\delta_{\mathrm{dist}}$, as described in Sec.~\ref{tsppt:subsec:goal_cluster}.
We varied one parameter at a time while holding the other at its nominal value, where nominal values mean the one we used in the previous experiments.
Specifically, we varied $\delta \in \{1,2,3,4,5\}$~m while fixing $\delta_{\mathrm{dist}}=5$~m, and varied $\delta_{\mathrm{dist}} \in \{4,5,6\}$~m while fixing $\delta=3$~m.
We tested on the aforementioned misleading-prior scenario and repeated the same $10$ trials used in the previous experiments.

\begin{table}[t]
    \centering
    \caption{Sensitivity to the bandwidth parameter $\delta$ with a fixed $\delta_{\mathrm{dist}}=5$ in the misleading-prior scenario. Mission duration is averaged over successful trials.}
    \label{tsppt:tab:sens_bandwidth}
    \begin{tabular}{l|c|c|c}
        \hline
        \textbf{$\mathbf{\delta}$ (m)} & \textbf{Success ($/10$)} & \textbf{Mission Duration (s)} & \textbf{Deviation (\%)} \\
        \hline
        $1$ & $10$ & $94.7$ & $+4.4$ \\
        \hline
        $2$ & $10$ & $95.7$ & $+5.5$ \\
        \hline
        $3$ & $10$ & $\mathbf{90.7}$ & $\mathbf{0.0}$ \\
        \hline
        $4$ & $10$ & $77.7$ & $-14.3$ \\
        \hline
        $5$ & $10$ & $91.8$ & $+1.2$ \\
        \hline
    \end{tabular}
\end{table}

\begin{table}[t]
    \centering
    \caption{Sensitivity to the cluster-merge distance $\delta_{\mathrm{dist}}$ with a fixed $\delta=3$ in the misleading-prior scenario. Mission duration is averaged over successful trials.}
    \label{tsppt:tab:sens_merge}
    \begin{tabular}{l|c|c|c}
        \hline
        \textbf{$\mathbf{\delta_{\mathrm{dist}}}$ (m)} & \textbf{Success ($/10$)} & \textbf{Mission Duration (s)} & \textbf{Deviation (\%)} \\
        \hline
        $4$ & $10$ & $93.9$ & $+3.5$ \\
        \hline
        $5$ & $10$ & $\mathbf{90.7}$ & $\mathbf{0.0}$ \\
        \hline
        $6$ & $10$ & $91.7$ & $+1.1$ \\
        \hline
    \end{tabular}
\end{table}

As shown in Table~\ref{tsppt:tab:sens_bandwidth} and ~\ref{tsppt:tab:sens_merge}, overall, the performance in terms of mission duration does not fluctuate much when the parameters vary, with a median percentage of $2.35\%$.
However, there are also cases where the mission duration varies a lot, such as $\delta=4$ in Table~\ref{tsppt:tab:sens_bandwidth}.
The possible reason behind is that the clustering process induces an abrupt change in the structure of the planning graph: when $\delta$ reaches certain values, the merging of frontier clusters can alter the connectivity of the \abbrTSPPT graph, which then affects the ordering of candidate search regions and the planned path, resulting in big deviation in the mission duration.

\subsection{Real-Robot Target Search in Unknown Environments}
\label{tsppt:subsec:real-robot}
We further validated the HATS-U system through deployment in real-world scenarios. To demonstrate the system's robustness across different platforms, experiments were conducted in two distinct environments:
\begin{itemize}
\item Corridor Setting: A spacious corridor area of size $40 \times 30$~m, whose point cloud map was shown in Fig.~\ref{tsppt:fig:Result-HATS-U-Dog}. The target was a bag, and the mission was executed by a quadrupedal robot.
\item Laboratory Setting: A $30 \times 20$~m confined, cluttered laboratory environment, whose point cloud map is shown in Fig.~\ref{tsppt:fig:intro}.
The designated target was a trash bin, and the search task was executed by a wheeled mobile robot.
\end{itemize}
For perception, we employed the YOLO~\cite{yolo11_ultralytics} object detector. The model was trained on a custom dataset that was manually collected and annotated.

In both scenarios, the robot successfully completed the search tasks with mission durations of $152$ and $104$ seconds, respectively. The search processes are visualized in Fig.~\ref{tsppt:fig:Result-HATS-U-Dog} and Fig.~\ref{tsppt:fig:Result-HATS-U-Car}.
These snapshots illustrate the real-time evolution of the occupancy grid and the corresponding camera views during the target search in the Corridor and Laboratory settings, respectively.
Different from exploration which often uses a $360^\circ$ Field Of View (FOV) Lidar for mapping, our target search {\blue is constrained by the limited FOV of a camera.}
As a result, the robot has to swing left and right to find the target during its motion using the camera.

Additionally, the YOLO software we use can return a confidence score between $[0,1]$ for each object detection.
We denote the sliding window duration as $T_{\mathrm{YOLO}} = 2.0$~s, the required fraction of confident detections as $P_{\mathrm{YOLO}} = 0.6$, and the detection confidence threshold as $\tau_{\mathrm{YOLO}} = 0.6$.
The robot stops if, within a sliding time window of $T_{\mathrm{YOLO}}$, more than $P_{\mathrm{YOLO}}$ of the detections have confidence scores exceeding $\tau_{\mathrm{YOLO}}$.
By doing so, the robot can reliably detect the object in practice.
Too small a threshold or too short a sliding window duration may lead to false detection and terminate the robot motion earlier.
On the flip side, YOLO may never detect the object and the robot may never terminate if we set too high a threshold or too long a period.
More details can be found in our video attachment.

\begin{figure}[tb]
	\centering
	\includegraphics[width=0.92\linewidth]{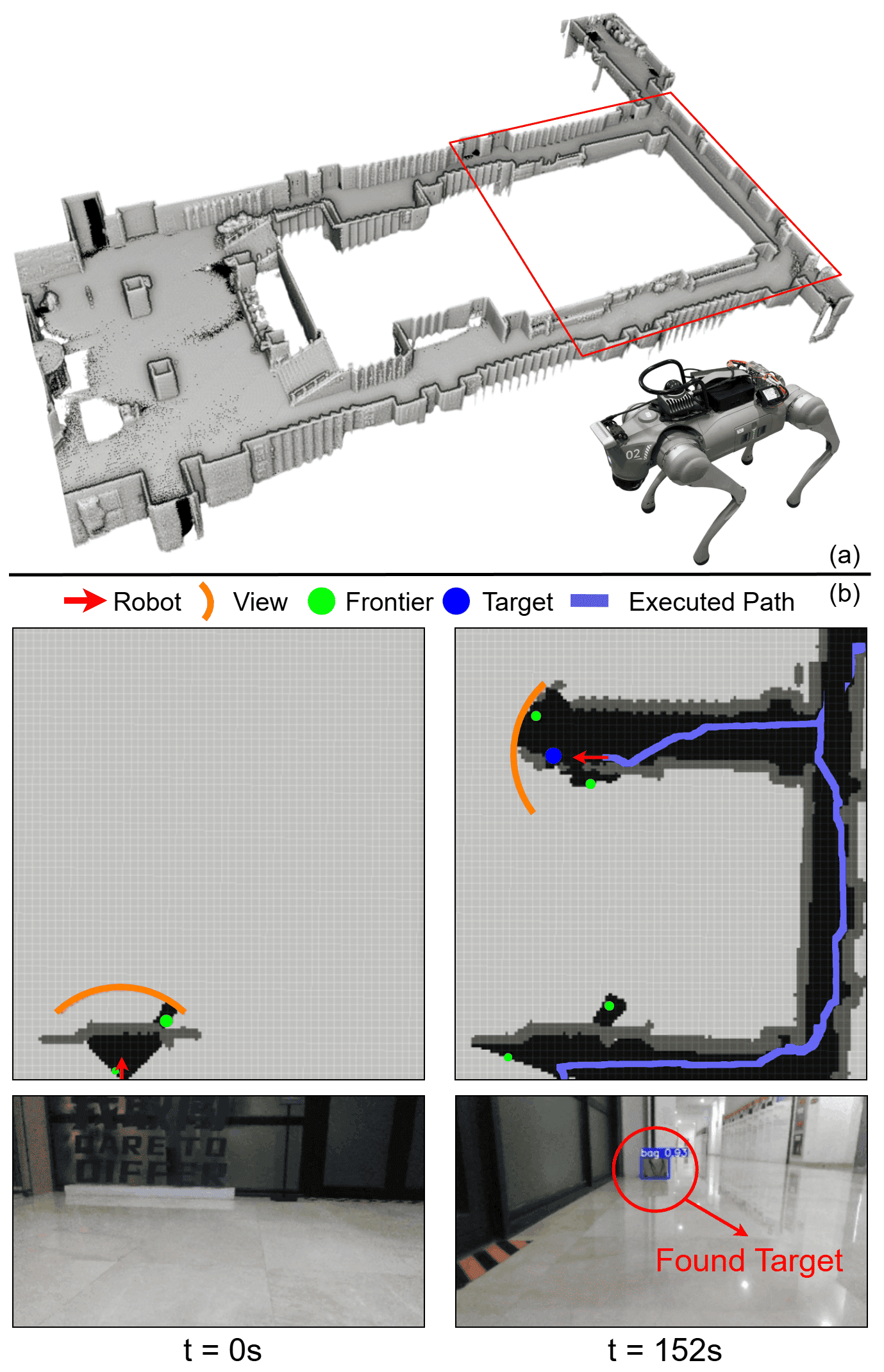}
	\caption{Experiment of HATS-U in a corridor setting using a quadruped robot. (a) Panoramic view of the environment and the robot. (b) Snapshots of the occupancy grid and camera image during the search process.}
    \label{tsppt:fig:Result-HATS-U-Dog}
\end{figure}

\begin{figure}[tb]
	\centering
    \includegraphics[width=0.92\linewidth]{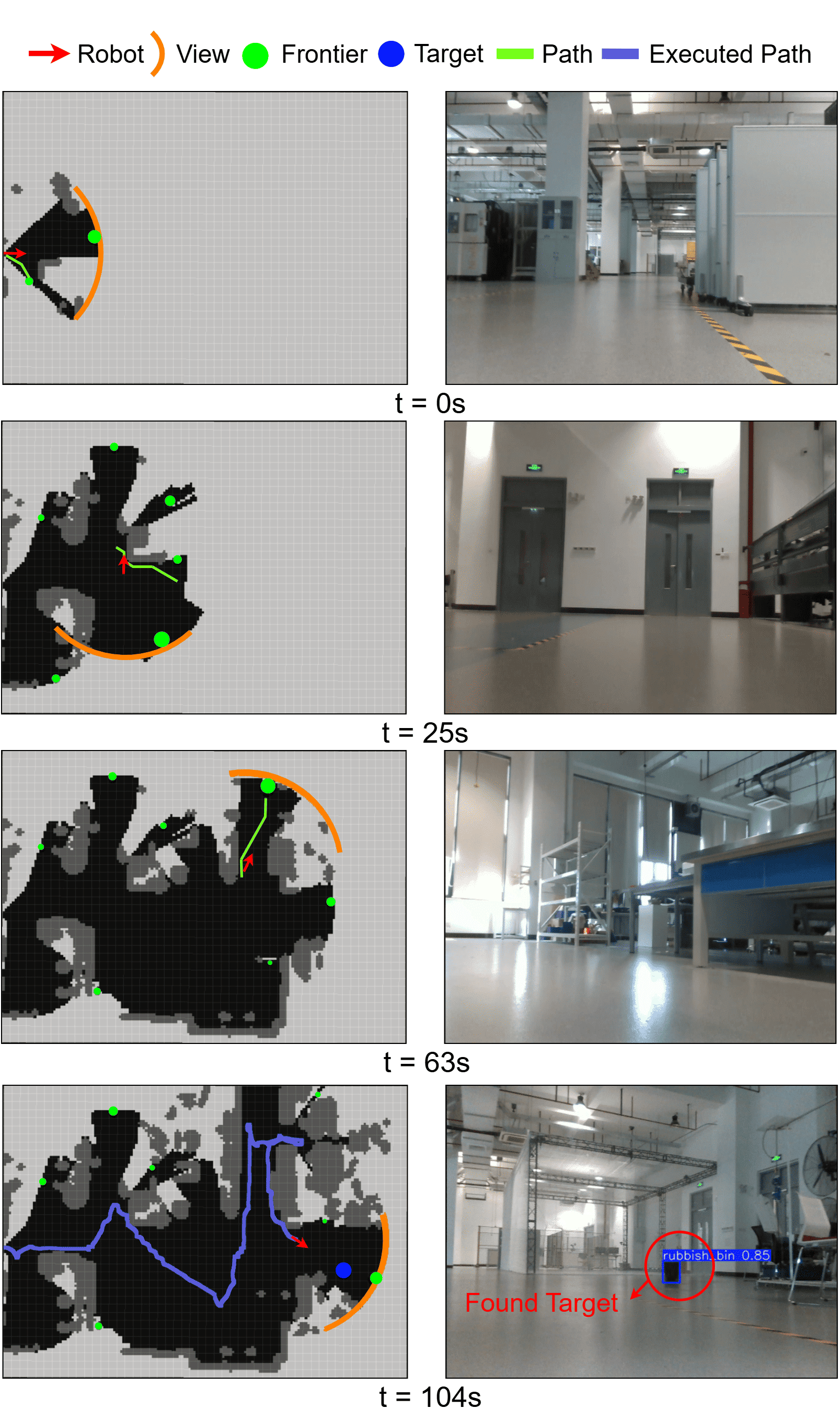}
	\caption{Snapshots of the occupancy grid and camera image during the target search process using the wheeled robot in a lab setting.}
    \label{tsppt:fig:Result-HATS-U-Car}
    \vspace{-5mm}
\end{figure}

%% file: conclusion.tex
In this paper, we introduced the Hamiltonian Path Problem with Probabilistic Terminals (\abbrTSPPT), a new variant of the \abbrTSP motivated by autonomous robot target search task.
To address the problem, we proposed several algorithms including Integer Programming, and heuristic search Routing with Probabilistic Terminals Star (\abbrRPT) and its bounded sub-optimal version.
Extensive tests on TSPLIB and synthetic datasets demonstrated that \abbrRPT outperforms baseline methods.
Furthermore, we integrated these algorithms into a system, Hierarchical Autonomous Target Search (\abbrHATS), capable of handling both lifelong search in known environments (\abbrHATS-L) and exploratory search in unknown environments (\abbrHATS-U).
The system was validated through simulations and physical deployment on both wheeled and quadrupedal robots.
Our \abbrRPT planner is able to balance between exploitation and exploration during the target object search.

For future work, we plan to extend this research from single-robot to multi-robot.
Besides, we can leverage semantic mapping techniques~\cite{ginting2024seek} to extract probability of vertices based on semantic information.
Finally, our \abbrRPT has the potential to be extended to track moving targets by incorporating a probability transition model in the Bayesian filter framework when targets can move.

%% file: appendix_a.tex
\section{Exact Baselines for \abbrTSPPT}
\label{TSPPT:app:exact}
\subsection{Integer Programming}
\label{TSPPT:app:ip}
We begin with a straightforward naive formulation and then improve it using auxiliary variables.

\subsubsection{Naive Formulation}
Let the vertices in $V$ be organized in a fixed order.
Let $x_{k,i} \in \{0,1\}$ denote a binary variable, where $k = 1,\cdots,N$ stands for the visiting order of vertex $v_i \in V$ in the solution.
There are in total $N^2$ binary variables, and let $X \in \{0,1\}^{N\times N}$ denote a square binary matrix consisting of $x_{k,i}$, with the $k$-th column being $x_k=[x_{k,1},x_{k,2},\cdots,x_{k,N}]$.
Each $x_k$ has only one element taking value $1$ with all other elements taking value $0$, indicating which vertex is visited as the $k$-th step in the robot's path.
We now define a few column vectors.
Let $p = [p(v_1),p(v_2),\cdots,p(v_N)]$ denote the probability vector.
Let $e$ denote a vector of length $N$ with all elements being one.
Let $C \in \mathbb{R}^{N\times N}$ denote a cost matrix with the entry in the $i$-th row and the $j$-th column being the edge cost $c(v_i,v_j)$.

To formulate the expected path cost, we introduce the following variable $l_k$ to describe the path length when the robot visits the $(k+1)$-th vertex in its path.
\begin{align}
    l_k = \sum_{i=1}^{k} x_i^T C x_{i+1}, k=1,2,\cdots,N-1
\end{align}
Note that $x_i^T C x_{i+1}$ selects the edge that is traversed by the robot when moving from the $i$-th vertex to the $(i+1)$-th vertex in its path.
Then, the objective function, i.e., the expected path cost, to be minimized is:
\begin{align}
&\xi(x) = (1-p^T x_1)p^Tx_2 l_1 + (1-p^T x_1)(1-p^T x_2)p^Tx_3 l_2 \nonumber\\
     & + \cdots + \prod_{j=1}^{N-1}(1-p^T x_j)l_{N-1} \\
&= \sum_{k=1}^{N-2} \left( \prod_{j=1}^{k} (1-p^T x_j) p^Tx_{k+1} l_k \right) + \prod_{j=1}^{N-1}(1-p^T x_j)l_{N-1} \nonumber
\end{align}
The constraints are:
\begin{align}
e^T X &= e^T, k=1,2,\cdots,N\label{cstr:layer}\\
Xe &= e, k=1,2,\cdots,N\label{cstr:unique}\\
x_1&=x_s.\label{cstr:init}
\end{align}
Here, Constraint (\ref{cstr:layer}) ensures that only one vertex is visited as the $k$-th step $k=1,2,\cdots,N$, and Constraint (\ref{cstr:unique}) ensures all selected vertices are unique and there is no repeated visit of any vertex.
Finally, $x_s$ is a binary vector of length $N$ with only one component being $1$ corresponding to the start vertex $v_s$, and Constraint (\ref{cstr:init}) requires the first selected vertex must be the start vertex.
The entire integer program is:
\begin{align*}
    \min_{x} &\;\;\xi(x)\\
    \text{subject to }& (\ref{cstr:layer}) (\ref{cstr:unique}) (\ref{cstr:init}), 
\end{align*}
which is a nonlinear program due to the product terms.

\subsubsection{Transformed Formulation}
This naive formulation has a complicated objective function $\xi(x)$ with long product terms.
We propose the following transformed method to simplify these terms using an idea similar to the one in Sec.~\ref{TSPPT:sec:method}.
We introduce a variable $q_k$ to denote the probability of not finding a target at any vertex among the first $k$ vertices along the path. 
\begin{align}
    q_{k}= \prod_{j=1}^{k}(1-p^T x_j)
\end{align}

With $q_k$, the expected path cost can be transformed into:
\begin{align}
    &\xi(x) = q_{1}x_1^T C x_{2} + q_{2}x_2^T C x_{3} + \dots +q_{N-1}x_{N-1}^T C x_N \nonumber\\
    &= \sum_{i=1}^{N-1} q_{i}x_i^T C x_{i+1}
\end{align}

The constraints for this new formulation remain the same as those defined for the naive formulation (i.e., Constraints (\ref{cstr:layer}), (\ref{cstr:unique}), and (\ref{cstr:init})).
The program is still nonlinear.

{\blue
\subsection{Held-Karp Dynamic Programming}
\label{TSPPT:app:dp}

The additive form of the expected cost in Eq.~\eqref{tsppt:eq:expected_cost} allows Held-Karp dynamic programming to be applied for HPP-PT. Let $S\subseteq V$ be a subset of visited vertices satisfying $v_s\in S$, and let $v\in S$ be the current endpoint. Define
\begin{align}
    q(S):=\prod_{w\in S}\bigl(1-p(w)\bigr),
\end{align}
which is the probability that the robot has not terminated at any vertex in $S$.
Let $D(S,v)$ denote the minimum expected cost of any path that starts at $v_s$, visits all the vertices in $S$, and ends at $v$. The base case is
\begin{align}
    D(\{v_s\},v_s)=0.
\end{align}
All infeasible states are assigned value $+\infty$. For $|S|>1$ and $v\in S\setminus\{v_s\}$, the recurrence is
\begin{gather}
D(S,v) = \min_{u\in S\setminus\{v\}} \left[ D(S\setminus\{v\},u) + q(S\setminus\{v\}) c(u,v) \right], \notag \\
q(S) = q(S\setminus\{v\})(1-p(v)).
\label{eq:held_karp_hpppt}
\end{gather}
The factor $q(S\setminus\{v\})$ appears because the edge $(u,v)$ is traversed only if the robot has not terminated at any previously visited vertex.
The optimal expected cost is
\begin{align}
\min_{v\in V\setminus\{v_s\}} D(V,v).
\label{eq:held_karp_solution}
\end{align}
}

%% file: appendix_b.tex
\section{System Hyperparameters}
\label{TSPPT:app:hyperparam}

Table~\ref{tab:hats_hyperparameters} summarizes the main hyperparameters used by \abbrHATS. Unless otherwise stated, the same hyperparameters are used for all methods within a given experimental setting.     

\begin{table*}[t]
\centering
\caption{Main \abbrHATS hyperparameters.}
\label{tab:hats_hyperparameters}
\begin{tabular}{lll}
\toprule
Parameter & Meaning & Usage \\
\midrule
$\alpha_1$ & True-positive rate in the Bayesian observation model & HATS-L (in Sec.~\ref{TSPPT:sec:HATS-L}); set to $0.8$ \\
$\alpha_2$ & False-positive rate in the Bayesian observation model & HATS-L (in Sec.~\ref{TSPPT:sec:HATS-L}); set to $0.4$ \\
$p_H$ & Belief threshold for declaring target-present & HATS-L (in Sec.~\ref{tsppt:subsec:HATS-L}); set to $0.98$ \\
$p_L$ & Belief threshold for declaring target-absent & HATS-L (in Sec.~\ref{tsppt:subsec:HATS-L}); set to $0.02$ \\
$w_u$ & Weight on unknown-space density $\phi_u$ & HATS-U (in Sec.~\ref{TSPPT:sec:HATS-U}); set to $0.3$ \\
$w_g$ & Weight on geometric visibility $\phi_g$ & HATS-U (in Sec.~\ref{TSPPT:sec:HATS-U}); set to $0.3$ \\
$w_o$ & Weight on object-prior field $\phi_o$ & HATS-U (in Sec.~\ref{TSPPT:sec:HATS-U}); set to $0.4$ \\
$\delta_c$ & Mean Shift convergence threshold & HATS-U (in Sec.~\ref{TSPPT:sec:HATS-U}); set to $0.1$~m \\
$\delta$ & Mean Shift bandwidth & HATS-U (in Sec.~\ref{TSPPT:sec:HATS-U}); swept from $1.0$~m to $5.0$~m \\
$\delta_{\mathrm{dist}}$ & Cluster-merging distance threshold & HATS-U (in Sec.~\ref{TSPPT:sec:HATS-U}); swept from $4.0$~m to $6.0$~m \\
$T_{\mathrm{YOLO}}$ & Detection sliding-window duration & HATS-U (in Sec.~\ref{tsppt:subsec:real-robot}); set to $2.0$~s \\
$\rho_{\mathrm{YOLO}}$ & Required fraction of confident detections & HATS-U (in Sec.~\ref{tsppt:subsec:real-robot}); set to $0.6$ \\
$\tau_{\mathrm{YOLO}}$ & YOLO confidence threshold & HATS-U (in Sec.~\ref{tsppt:subsec:real-robot}); set to $0.6$ \\
\bottomrule
\end{tabular}
\end{table*}